\documentclass[letterpaper,11pt]{article}
\usepackage[margin=1in]{geometry}

\usepackage[utf8]{inputenc}
\usepackage[T1]{fontenc}
\usepackage{kpfonts}
\usepackage{microtype}

\usepackage{diagbox}
\usepackage{makecell}
\usepackage{booktabs}
\usepackage{amsmath}
\usepackage{amsfonts}
\usepackage{amssymb}
\usepackage{amsthm}
\usepackage{url,ifthen}
\usepackage{enumitem}
\usepackage{srcltx}
\usepackage{multirow}
\usepackage{boxedminipage}
\usepackage{nicefrac}
\usepackage{xspace}
\usepackage{graphicx}
\usepackage{color}
\usepackage{colortbl}
\usepackage{setspace}
\usepackage[numbers]{natbib}
\setcitestyle{square, authoryear}

\usepackage{amsmath,amsthm,amsfonts}

\usepackage{caption}
\usepackage[pdftex]{hyperref}
\hypersetup{
    unicode=false,          
    pdftoolbar=true,        
    pdfmenubar=true,        
    pdffitwindow=false,      
    pdfnewwindow=true,      
    colorlinks=true,       
    linkcolor=DarkBlue,          
    citecolor=DarkGreen,        
    filecolor=DarkRed,      
    urlcolor=DarkBlue,          
    pdftitle={},
    pdfauthor={},
}

\newcommand{\argmax}{\mathop{\rm argmax}}

\definecolor{DarkGreen}{rgb}{0.1,0.5,0.1}
\definecolor{DarkRed}{rgb}{0.5,0.1,0.1}
\definecolor{DarkBlue}{rgb}{0.1,0.1,0.5}
\definecolor{Gray}{rgb}{0.2,0.2,0.2}

\newcommand{\argmin}{{\rm argmin}}
\newcommand{\thetaPO}{\theta_{\mathrm{PO}}}

\DeclareMathOperator{\cD}{\mathcal D}
\newcommand{\E}{\mathbb{E}}

\newcommand{\PR}{\mathrm{PR}}
\newcommand{\DPR}{\mathrm{DPR}}
\newcommand{\PRhat}{\widehat{\PR}}
\newcommand{\cO}{{\mathcal{O}}}
\newcommand{\cS}{{\mathcal{S}}}
\newcommand{\cP}{{\mathcal{P}}}
\newcommand{\cE}{{\mathcal{E}}}
\newcommand{\cA}{{\mathcal{A}}}

\newcommand{\PRLB}{{\mathrm{PR}_\mathrm{LB}}}

\newcommand{\PRmin}{{\mathrm{PR}}_{\min}}
\newcommand{\Rph}{{\mathrm{Reg}}_{\mathrm{ph}}}
\newcommand{\Reg}{\mathrm{Reg}}

\usepackage{subfig}

\usepackage{algorithm}
\usepackage{algpseudocode}

\usepackage[margin=1in]{geometry}

\usepackage[utf8]{inputenc}
\usepackage[T1]{fontenc}
\usepackage{microtype}
\usepackage{csquotes}
\RequirePackage[scaled=0.90]{helvet}
\RequirePackage[scaled=0.85]{beramono}
\RequirePackage{textcomp}

\newtheorem{theorem}{Theorem}
\newtheorem{definition}{Definition}
\newtheorem{proposition}{Proposition}

\newtheorem{remark}{Remark}
\newtheorem{lemma}[theorem]{Lemma}
\newtheorem{example}{Example}

\newtheorem{assumption}{Assumption}

\definecolor{mydarkblue}{rgb}{0,0.08,0.45}
\hypersetup{ %
  colorlinks=true,
  linkcolor=mydarkblue,
  citecolor=mydarkblue,
  filecolor=mydarkblue,
  urlcolor=mydarkblue}
\setcitestyle{authoryear,round,citesep={;},aysep={,},yysep={;}}

\title{Regret Minimization with Performative Feedback\vspace{0.5cm}}
\author{Meena Jagadeesan, Tijana Zrnic, Celestine Mendler-D\"unner}
\date{}

\makeatletter
\newcommand{\printfnsymbol}[1]{%
  \textsuperscript{\@fnsymbol{#1}}%
}
\makeatother

\usepackage[skip=2mm,indent=5mm]{parskip}

\usepackage{authblk}
\stepcounter{footnote}
\addtocounter{footnote}{-1}
\author[1]{Meena Jagadeesan}
\author[1]{Tijana Zrnic}
\author[2]{Celestine Mendler-Dünner}
\affil[1]{University of California, Berkeley}
\affil[2]{Max Planck Institute for Intelligent Systems, Tübingen}
\date{}                     
\begin{document}

\maketitle

\begin{abstract}
In performative prediction, the deployment of a predictive model triggers a shift in the data distribution. As these shifts are typically unknown ahead of time, the learner needs to deploy a model to get feedback about the distribution it induces. We study the problem of finding near-optimal models under performativity while maintaining low regret. 
On the surface, this problem might seem equivalent to a bandit problem. However, it exhibits a fundamentally richer feedback structure that we refer to as \textit{performative feedback}: after every deployment, the learner receives samples from the shifted distribution rather than only bandit feedback about the reward. Our main contribution is an algorithm that achieves regret bounds scaling only with the complexity of the distribution shifts and not that of the reward function. The algorithm only relies on smoothness of the shifts and does not assume convexity. Moreover, its final iterate is guaranteed to be near-optimal.
The key algorithmic idea is careful exploration of the distribution shifts that informs a novel construction of confidence bounds on the risk of unexplored models.
 More broadly, our work establishes a conceptual approach for leveraging tools from the bandits literature for the purpose of regret minimization with performative feedback. 
\end{abstract}

\section{Introduction}
Predictive models deployed in social settings are often performative. This means that the model's predictions---by means of being used to inform consequential decisions---influence the outcomes the model aims to predict in the first place. For example, travel time estimates influence routing decisions and thus realized travel times, stock price predictions influence trading activity and hence prices. Such feedback-loop behavior arises in a variety of domains, including public policy, trading, traffic predictions, and recommendation systems.

\citet{perdomo20pp} formalized this phenomenon under the name \emph{performative prediction}.
A key concept in this framework is the \emph{distribution map}, which formalizes the dependence of the data distribution on the deployed predictive model. This object maps a model, encoded by a parameter vector $\theta$, to a distribution $\cD(\theta)$ over instances.
Naturally, in a performative environment, a model's performance is measured on the distribution that results from its deployment. 
That is, given a loss function $\ell(z;\theta)$, which measures the learner's loss when they predict on instance $z$ using model $\theta$, we evaluate a model based on its \emph{performative risk}, defined as
\begin{equation}
\label{eq:obj}
\PR(\theta) := \mathrm E_{z\sim\cD(\theta)} \,\ell(z;\theta).
\end{equation}
In contrast with the risk function studied in classical supervised learning, the performative risk takes an expectation over a model-dependent distribution. Importantly, this distribution is \emph{unknown} ahead of time; for example, one can hardly anticipate the distribution of travel times induced by a traffic forecasting system without deploying the system first.

Due to this inherent uncertainty about $\cD(\theta)$, it is not possible to find a model with low performative risk offline. The learner needs to interact with the environment and deploy models $\theta$ to explore the induced distributions $\cD(\theta)$. Given the online nature of this task, we measure the loss incurred by deploying a sequence of models $\theta_1,\dots,\theta_T$ by evaluating the \emph{performative regret}:
\begin{equation*}
\Reg(T) := \sum_{t=1}^T \left(\E\, \PR(\theta_t) -  \min_{\theta} \PR(\theta)\right),
\end{equation*}
where the expectation is taken over the possible randomness in the choice of $\{\theta_t\}_{t=1}^T$.
Performative regret measures the suboptimality of the deployed sequence of models relative to a \emph{performative optimum} 
$\thetaPO\in \argmin_\theta \PR(\theta)$.

At first glance, performative regret minimization might seem equivalent to a classical bandit problem. Bandit solutions minimize regret while requiring only noisy zeroth-order access to the unknown reward function---in our case $\PR$. The resulting regret bounds generally grow with some notion of complexity of the reward function.

However, a naive application of bandit baselines misses out on a crucial fact: performative regret minimization exhibits significantly richer feedback than bandit feedback. When deploying a model $\theta$, the learner gains access to samples from the induced distribution $\cD(\theta)$, rather than only a noisy estimate of the risk $\PR(\theta)$. We call this feedback model \emph{performative feedback}. Together with the fact that the learner knows the loss $\ell(z;\theta)$, performative feedback can be used to inform the reward of unexplored arms. For instance, it allows the computation of an unbiased estimate of $\E_{z\sim \cD(\theta)} \,\ell(z;\theta')$ for \emph{any} point $\theta'$. 

To illustrate the power of this feedback model, consider the limiting case in which the performative effects entirely vanish and the distribution map is constant, i.e. $\cD(\theta)\equiv \cD_*$ for some fixed distribution $\cD_*$ independent of $\theta$. With zeroth-order feedback, the learner would still need to deploy different models to explore the landscape of $\PR$ and find a point with low risk. However, with performative feedback, a \emph{single} deployment gives samples from $\cD_*$, thus resolving all uncertainty in the objective~\eqref{eq:obj} apart from finite-sample uncertainty. This raises the question:
\emph{with performative feedback, can one achieve regret bounds that scale only with the complexity of the distribution map, and not that of the performative risk?}

\subsection{Our contribution} 


We study the problem of performative regret minimization based on performative feedback.
Our main contribution is performative regret bounds that scale primarily with the complexity of the distribution map. The key conceptual idea is to apply bandits tools to carefully explore the distribution map, and then propagate this knowledge to the objective~\eqref{eq:obj} in order to minimize performative regret.

\paragraph{Performative confidence bounds algorithm.} Our main focus is on a setting where the distribution map is Lipschitz in an appropriate sense. 
We propose a new algorithm that takes advantage of performative feedback in order to construct non-trivial confidence bounds on the performative risk in unexplored regions of the parameter space and thus guide exploration. A crucial implication of these bounds is that the algorithm can discard highly suboptimal regions of the parameter space without ever deploying a model nearby.
We summarize the regret guarantee of our \emph{performative confidence bounds} algorithm:

\begin{theorem}[Informal]
\label{thm:informal}
Suppose that the distribution map $\cD(\theta)$ is $\epsilon$-Lipschitz and that the loss $\ell(z;\theta)$ is $L_z$-Lipschitz in $z$. Then,  after $T$ deployments, the performative confidence bounds algorithm achieves a regret bound of
\[\Reg(T) = \tilde\cO\left(\sqrt{T}+T^{\frac{d+1}{d+2}}(L_z\epsilon)^{\frac{d}{d+2}}\right),\] 
where $d$ denotes the zooming dimension of the problem.
\end{theorem}

We compare the bound in Theorem \ref{thm:informal} to a baseline Lipschitz bandits regret bound. The concept of zooming dimension stems from the work of~\citet{kleinberg2008multi} and serves as an instance-dependent notion of dimensionality. \citet{kleinberg2008multi} showed that sublinear regret $\tilde\cO\big(T^{\frac{d'+1}{d'+2}} L^{\frac{d'}{d'+2}}\big)$ can be achieved if the reward function is $L$-Lipschitz, where $d'$ is a zooming dimension. The performative risk can be guaranteed to be Lipschitz if the distribution map is Lipschitz and the loss $\ell(z;\theta)$ is Lipschitz in \textit{both} arguments.

The primary benefit of Theorem \ref{thm:informal} is that our regret bound scales only with the Lipschitz constant of the distribution map, rather than the Lipschitz constant of the performative risk. In particular, our result allows $\PR(\theta)$ to be highly irregular as a function of $\theta$, seeing that $\ell(z;\theta)$ as a function of $\theta$ is unconstrained. 
This difference becomes salient when $\epsilon\rightarrow 0$, meaning that the performative effects vanish:
our regret bound grows as $\tilde\cO(\sqrt{T})$ in an essentially  dimension-independent manner. More precisely, the dimension can only arise implicitly through a model class complexity term. 
On the other hand, the rate of classical Lipschitz bandits remains exponential in the dimension. 

Another difference between our regret bound and that of Lipschitz bandits is in the zooming dimension. In particular, $d'$ is a zooming dimension \emph{no smaller than} the zooming dimension we obtain in Theorem~\ref{thm:informal}.
As we will elaborate on in later sections, the benefit we derive from the zooming dimension comes from the fact that it implicitly depends on the Lipschitz constant driving the objective, which is smaller when making full use of performative feedback. 

\paragraph{Extension to location families.} In addition, we study performative regret minimization for the special case where the distribution map has a \textit{location family} form \citep{miller21echo}. We again prove regret bounds that scale only with the complexity of the distribution map, rather than the complexity of the performative risk. We adapt the LinUCB algorithm \citep{li2010contextual} to learn the hidden parameters of the location family. This enables us to achieve a $\tilde \cO(\sqrt{T})$
regret without placing any strong convexity assumptions on the performative risk that are required in \citep{miller21echo}. In particular, our result again allows $\PR(\theta)$ to be highly irregular as a function of $\theta$.

\paragraph{Consequences for finding performative optima.} While we have contextualized our work within online regret minimization, our performative confidence bounds algorithm has the additional property that it converges to the set of performative optima. Thus, if run for sufficiently many time steps, it generates a model with near-minimal performative risk: in particular, a model with risk at most $\tilde\cO\left(T^{-\frac{1}{d+2}} (L_z \epsilon)^{\frac{d}{d+2}} \right)$ greater than the minimum performative risk $\min_{\theta} \PR(\theta)$. 

More broadly, our work establishes a connection between performative prediction and the bandits literature, which we believe is a worthwhile direction for future inquiry. 


\subsection{Related work}

\paragraph{Performative prediction.}
Prior work on performative prediction has largely studied gradient-based optimization methods \citep{perdomo20pp, mendler20stochasticPP, drusvyatskiy2020stochastic, brown2020performative, miller21echo, izzo2021learn, maheshwari2021zeroth, li2021state, raydecision, dong2021approximate}.
Many of the studied procedures only converge to \emph{performatively stable} points, that is, points $\theta$ that satisfy the fixed-point condition $\theta \in \argmin_{\theta'} \E_{z\sim \cD(\theta)} \ell(z;\theta')$.
In general, stable points are not minimizers of the performative risk \citep{perdomo20pp, miller21echo}, which implies that procedures converging to stable points do not achieve sublinear performative regret. There are exceptions in the literature that focus on finding performative optima \citep{miller21echo,izzo2021learn}, but those algorithms rely on proving or assuming convexity of the performative risk; in this work we make no convexity assumptions. In fact, it is known that the performative risk can be nonconvex even when the loss $\ell(z;\theta)$ is convex and the performative effects are relatively weak~\citep{perdomo20pp,miller21echo}. One other work that studies performative optimality, without imposing convexity, is that of \citet{dong2021approximate}, but they  focus on optimization heuristics that are not guaranteed to minimize performative regret.

\paragraph{Learning in Stackelberg games.} Performative prediction is closely related to learning in \emph{Stackelberg games}: if $\cD(\theta)$ is thought of as a best response to the deployment of $\theta$ according to some unspecified utility function, then performative optima can be thought of as Stackelberg equilibria. There have been many works on learning dynamics in Stackelberg games in recent years~\citep{balcan2015commitment, jin2020local, fiez2020implicit, fiez2020local}. 
Notably, \citet{balcan2015commitment} also study the benefit of a richer feedback model: they assume the agent's type is revealed after taking an action. When combined with a known agent-response model, this allows them to directly infer the loss of unexplored strategies. In contrast, performative feedback does not imply full-information feedback.
One instance of performative prediction that has an explicit Stackelberg structure, meaning $\cD(\theta)$ is defined as a best response, is \emph{strategic classification}~\citep{hardt16strat}.
Several works have studied learning dynamics in strategic classification \citep{dong2018strategic, chen2020learning, bechavod2021gaming, zrnic2021leads}; notably, \citet{dong2018strategic} and \citet{chen2020learning} provide solutions that minimize Stackelberg regret, of which performative regret is an analog in the performative prediction context. However, all of these works rely on strong structural assumptions, such as linearity of the predictor or convexity of the risk function, which significantly reduce the amount of necessary exploration compared to the mild Lipschitzness conditions we impose in our work.

\paragraph{Continuum-armed bandits.} Particularly inspiring for our work is the literature on \emph{continuum-armed bandits} \citep{agrawal1995continuum, kleinberg2004nearly, auer2007improved, kleinberg2008multi, podimata20lipschitzbandit}. As we will elaborate on in Section \ref{sec:blackbox}, performative prediction
can be cast as a Lipschitz continuum-armed bandit problem. However, while this means that one can use an off-the-shelf Lipschitz bandit algorithm to minimize performative regret, this would generally be a conservative solution. After ``pulling an arm'' $\theta$ in performative prediction the learner observes samples from $\cD(\theta)$. As explained earlier, in combination with the structure of our objective, this feedback model is more powerful than classical bandit feedback, where a noisy version of the mean reward at $\theta$ is observed. Moreover, 
it is fundamentally different from other partial-feedback and side-information models studied in the literature, e.g.~\citep{mannor11graphfeedback,kocak14sideobs,wu15sideinfo,alon16graph}.

\subsection{Preliminaries}

Performative prediction, set up as an online learning problem, can be formalized as follows. The learner chooses models $\theta$ in the parameter space $\Theta\subset \mathbb{R}^{d_\Theta}$. We assume\footnote{Throughout we use $\|\cdot\|$ to denote the $\ell_2$-norm for vectors and the operator norm for matrices.} $\max \{\|\theta\|:\theta\in\Theta\} \leq 1$ for simplicity. The expected loss of model $\theta$ is given by $\PR(\theta) = \mathrm E_{z\sim\cD(\theta)} \ell(z;\theta)$. 
We assume that the objective function is bounded so that $\ell(z; \theta) \in [0,1]$ for all $z$ and $\theta$. 

At every time step $t$, the learner chooses a model $\theta_t$ and observes a constant number $m_0$ of i.i.d. samples,
\[\{z_t^{(i)}\}_{i\in{[m_0]}}, \text{ where } z_t^{(i)}\sim\cD(\theta_t).\] 
The regret incurred by choosing $\theta_t$ at time step $t$ is $\Delta(\theta_t):=\PR(\theta_t) - \PR(\thetaPO)$, where $\thetaPO$ is the performative optimum. 

The constant $m_0$ quantifies how many samples the learner can collect in a time window determined by how often they incur regret. For example, at the beginning of each week the learner might update the model, and thus at the end of each week they incur regret for the model they chose to deploy. In that case, $m_0$ is the number of samples the learner collects per week.
Note that a learner with larger $m_0$ collects an empirical distribution that more accurately reflects $\cD(\theta_t)$ and thus naturally minimizes regret at a faster rate.

To formally disentangle the effects of the parameter vector $\theta$ on the performative risk through the distribution map and the loss function, we use the notion of the \emph{decoupled performative risk} \citep{perdomo20pp}:
\begin{equation*}
    \DPR(\theta,\theta') := \mathrm E_{z\sim \cD(\theta)}\, \ell(z;\theta').
\end{equation*}
This object captures the risk incurred by a model $\theta'$ on the distribution $\cD(\theta)$. Note that $\PR(\theta) = \DPR(\theta, \theta)$ by definition.

To measure the complexity of the distribution map we consider how much the distribution $\cD(\theta)$ can change with changes in $\theta$, as formalized by $\epsilon$\emph{-sensitivity}.

\begin{assumption}[$\epsilon$-sensitivity~\citep{perdomo20pp}]
A distribution map $\cD(\cdot)$ is $\epsilon$-sensitive if for any pair $\theta,\theta'\in\Theta$ it holds that 
\[\mathcal W (\cD(\theta),\cD(\theta'))\leq \epsilon \|\theta-\theta'\|,\]
where $\mathcal W$ denotes the Wasserstein-1 distance.
\label{ass:lipschitz}
\end{assumption}
\noindent In the context of a traffic forecasting app, $\epsilon$ can be thought of as being proportional to the size of the user base of the app. When $\cD(\theta)$ arises from the aggregate behavior of strategic agents manipulating their features in response to a model $\theta$, the sensitivity $\epsilon$ grows when features are more easily manipulable.

\section{A black-box bandits approach}

\label{sec:blackbox}

Performative regret minimization can be set up as a continuum-armed bandits problem where an arm corresponds to a choice of model parameters $\theta$. Performative feedback is sufficient to simulate noisy zeroth-order feedback about the reward function, as assumed in bandits. When we deploy $\theta_t$, the samples from $\cD(\theta_t)$ enable us to compute an unbiased estimate
\[\PRhat(\theta_t) = \frac 1 {m_0} \sum_{i=1}^{m_0}  \ell\big(z_t^{(i)};\theta_t\big)\] 
of the risk $\PR(\theta_t)$. Moreover, since we assume the loss function is bounded, the noise in the estimate $\PRhat(\theta_t)$ is subgaussian, as typically required in bandits.

A standard condition that makes continuum-armed bandit problems tractable is a bound on how fast the reward can change when moving from one arm to a nearby arm. Formally, this regularity is ensured by assuming Lipschitzness of the reward function---in our case, Lipschitzness of the performative risk.

The dependence of $\PR(\theta)$ on $\theta$ is twofold, as seen in Equation~\eqref{eq:obj}.
Thus, the most natural way to ensure that $\PR(\theta)$ is Lipschitz is to ensure that each of these two dependencies is Lipschitz. This yields the following bound: 
\begin{lemma}[Lipschitzness of $\PR$]
\label{lemma:lipschitz}
If the loss $\ell(z;\theta)$ is $L_z$-Lipschitz in $z$ and $L_\theta$-Lipschitz in $\theta$ and the distribution map is $\epsilon$-sensitive, then the performative risk is $(L_{\theta} + \epsilon L_z )$-Lipschitz.
\end{lemma}
The intuition behind Lemma \ref{lemma:lipschitz} is that $\PR(\theta)$ is guaranteed to be Lipschitz if $\DPR(\theta,\theta')$ is Lipschitz in each argument individually. Lipschitzness in the second argument follows from requiring that the loss be Lipschitz in $\theta$. Lipschitzness in the first argument follows from combining Lipschitzness of the loss in $z$ and $\epsilon$-sensitivity of the distribution map.

\subsection{Adaptive zooming} 
Once we have established Lipschitzness of the performative risk, we can apply techniques from the Lipschitz bandits literature. \citet{kleinberg2008multi} proposed a bandit algorithm that adaptively discretizes promising regions of the space of arms, using Lipschitzness of the reward function to bound the additional loss due to discretization. Their method, called the \emph{zooming algorithm}, will serve as a baseline for our problem. The algorithm enjoys an instance-dependent regret that takes advantage of nice problem instances, while maintaining tight guarantees in the worst case. The rate depends on the \emph{zooming dimension}, which is upper bounded in the worst case by the dimension of the full space $d_\Theta$.

\begin{proposition}[Zooming algorithm \citep{kleinberg2008multi}]
Suppose $m_0 = o(\log T)$. Then, after $T$ deployments, the zooming algorithm achieves a regret bound of
\[\Reg(T) = \cO\left(T^{\frac{d+1}{d+2}} \left(\frac{\log T}{m_0}\right)^{\frac{1}{d+2}} (L_\theta + \epsilon L_z)^{\frac{d}{d+2}}\right),\]
where $d$ denotes the $(L_\theta + \epsilon L_z)$-zooming dimension. 
\label{prop:lipschitz}
\end{proposition}

The zooming dimension quantifies the niceness of a problem instance by measuring the size of a covering of near-optimal arms, instead of the entire parameter space. Roughly speaking, if the reward function is very ``flat'' in that there are many near-optimal points, then the zooming dimension is close to the dimension $d_\Theta$ of the parameter space. However, if the reward has sufficient curvature, then the zooming dimension can be much smaller than $d_\Theta$.
The zooming dimension is defined formally as follows:

\begin{definition}[$\alpha$-zooming dimension]
\label{def:zdim}
A performative prediction problem instance has $\alpha$-zooming dimension equal to $d$ if any minimal $s$-cover of any subset of $\{\theta:\Delta(\theta)\leq 16 \alpha s\}$ includes at most a constant multiple of $(3/s)^{d}$ elements from $\{\theta: 16 \alpha r \le \Delta(\theta) < 32 \alpha r\}$, for all $0<r\leq s\leq 1$.
\end{definition}

For well-behaved instances, the definition intuitively requires every minimal $s$-cover of $\{\theta:16\alpha r \le \Delta(\theta) < 32 \alpha r\}$ to have size at most of order $(3/s)^{d}$.  Definition~\ref{def:zdim} slightly differs from the definition presented in~\citep{kleinberg2008multi} and makes the dependence on the Lipschitz constant explicit; we use Definition~\ref{def:zdim} to later ease the comparison to our new algorithm. The differences between the two definitions are minor technicalities that we do not expect to alter the zooming dimension in a meaningful way, neither formally nor conceptually.  See Appendix~\ref{app:zooming} for a discussion.

\section{Making use of performative feedback}
\label{sec:perf_feedback}

In this section, we illustrate how we can take advantage of performative feedback beyond computing a point estimate of the deployed model's risk.
For now, we ignore finite-sample considerations and assume access to the entire distribution $\cD(\theta)$ after deploying a model $\theta$. We will address finite-sample uncertainty when presenting our main algorithm in the next section.


\subsection{Constructing performative confidence bounds}

First, we demonstrate how performative feedback allows constructing tighter confidence bounds on the performative risk of unexplored models, compared to only relying on Lipschitzness of the risk function $\PR(\theta)$.

\begin{figure}[t!]
    \centering
    \subfloat[Baseline confidence bounds]
    {\includegraphics[height=0.21\columnwidth]{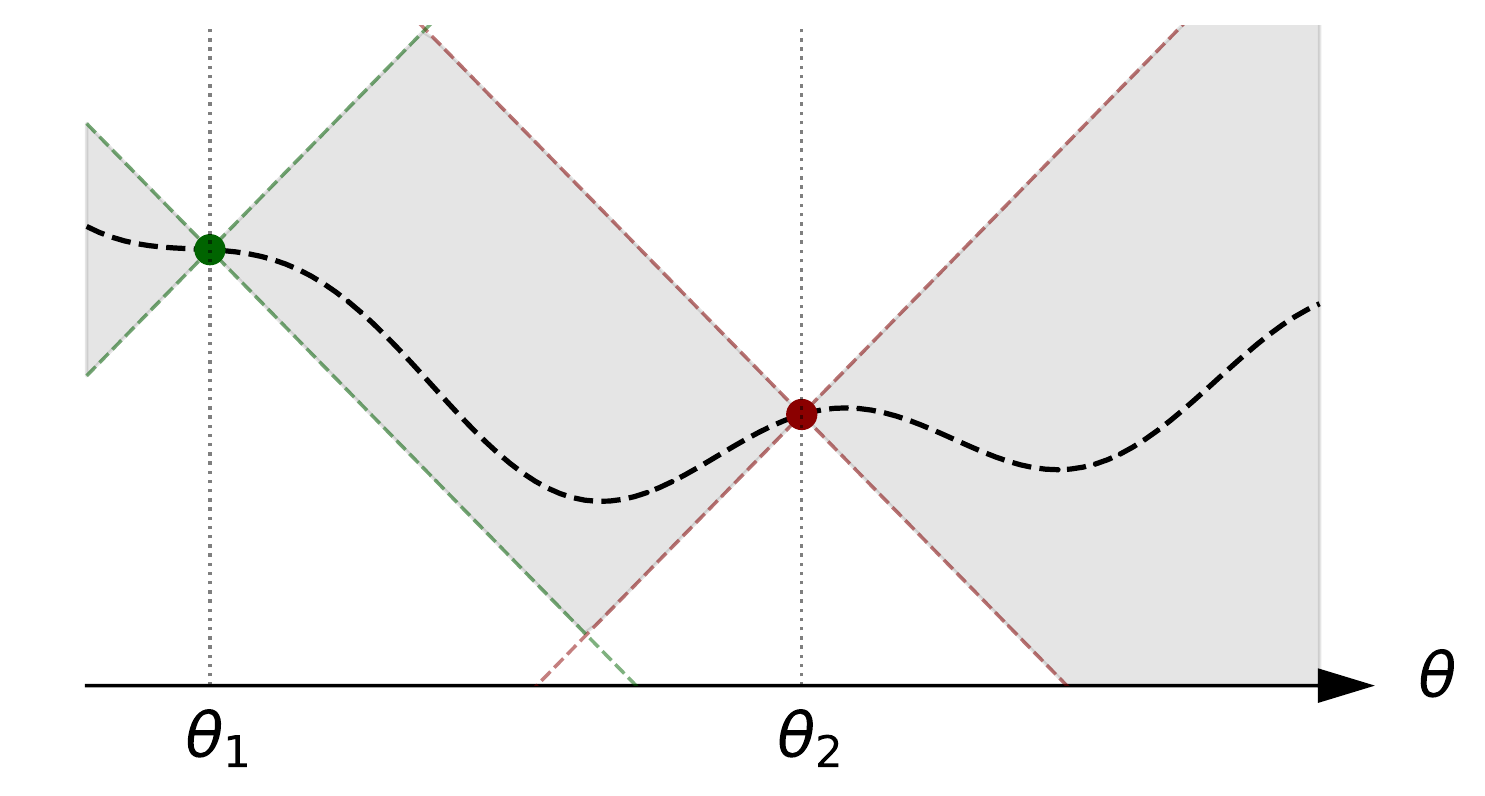}\label{fig:lipa}}\hspace{1cm}
    \subfloat[Performative confidence bounds $\quad\quad\quad\quad\quad$]
    {\includegraphics[height=0.21\columnwidth]{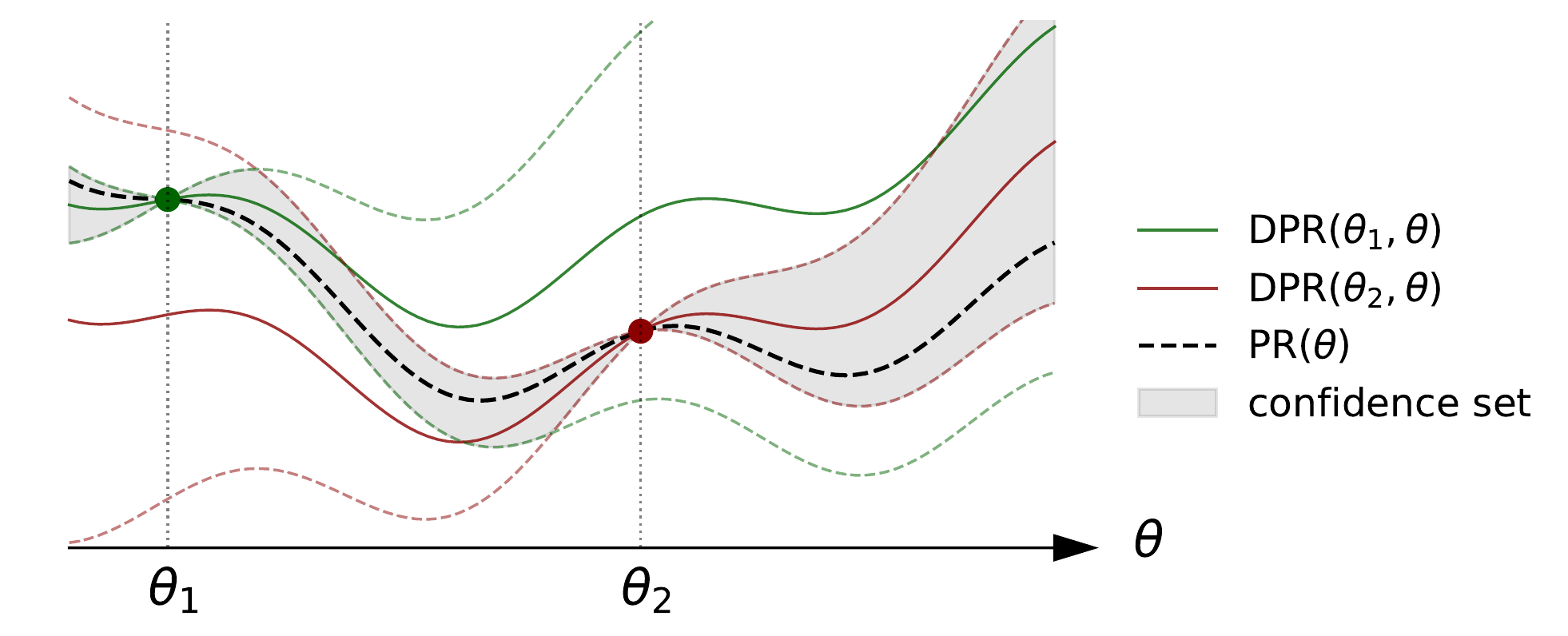}\label{fig:lipb}}
    \caption{Confidence bounds after deploying $\theta_1$ and $\theta_2$. (left) Confidence bounds via Lipschitzness, as stated in Equation~\eqref{eq:lipschitz_cb}. (right) Performative confidence bounds, as stated in Equation~\eqref{eq:tighter_lipschitz_cb}. The performative feedback model used for this illustration can be found in Appendix \ref{app:figures}.}
    \label{fig:confidence}
\end{figure}

Suppose we deploy a set of models $\cS\subseteq \Theta$ and for each $\theta\in\cS$ we observe $\cD(\theta)$.
Then, under the regularity conditions of Lemma \ref{lemma:lipschitz}, we can bound the risk of any $\theta'\in\Theta$ as 
\begin{align}
 \max_{\theta\in \cS}\;\PR(\theta)-(L_\theta + L_z\epsilon) \|\theta-\theta'\| \leq \PR(\theta')
  \leq \min_{\theta\in \cS}\;\PR(\theta)+(L_\theta + L_z\epsilon) \|\theta-\theta'\|.
\label{eq:lipschitz_cb}
 \end{align}


These confidence bounds only use $\cD(\theta)$ for the purpose of computing $\PR(\theta)$ and rely on Lipschitzness to construct confidence sets around the risk of unexplored models. However, in light of the structure of our objective function~\eqref{eq:obj}, the bounds in Equation~\eqref{eq:lipschitz_cb} do not make full use of performative feedback; in particular, access to $\cD(\theta)$ actually allows us to evaluate $\DPR(\theta,\theta')$ for \emph{any} $\theta'$. 
Importantly, this information can further reduce our uncertainty about $\PR(\theta')$, and we can bound: 
\begin{align*}
 \PR(\theta') &=  \DPR(\theta,\theta') +  \left(\DPR(\theta',\theta') - \DPR(\theta,\theta')\right)\\
 &\leq \DPR(\theta,\theta') + L_z\epsilon \|\theta-\theta'\|.
\end{align*}
Thus we can get tighter bounds on the performative risk at an unexplored parameter~$\theta'$:
\begin{align}
\label{eq:tighter_lipschitz_cb}
    \max_{\theta\in \cS} \;\DPR(\theta,\theta')-L_z\epsilon \|\theta -\theta'\|\leq \PR(\theta')
    \leq \min_{\theta\in \cS} \;\DPR(\theta,\theta')+L_z\epsilon \|\theta -\theta'\|.
\end{align}
\noindent We call the confidence bounds computed in \eqref{eq:tighter_lipschitz_cb} \emph{performative confidence bounds}.  In Figure~\ref{fig:confidence}, we visualize and contrast these confidence bounds with the confidence bounds obtained via Lipschitzness. We observe that by computing $\DPR$ we can significantly tighten the confidence regions.


The tightness of the confidence bounds depends on the set $\cS$ of deployed models. By choosing a cover of the parameter space, we can get an estimate of the performative risk that has low approximation error on the whole parameter space.

\begin{proposition}
\label{prop:cover}
Let $\cS_\gamma$ be a $\gamma$-cover of $\Theta$ and suppose we deploy all models $\theta\in \cS_\gamma$. Then, using performative feedback we can compute an estimate of the performative risk   $\PRhat(\theta)$ such that for any $\theta \in \Theta$ it holds that \[|\PR(\theta)-\PRhat(\theta)|\leq \gamma L_z\epsilon.\]
\end{proposition}

Proposition \ref{prop:cover} implies that after exploring the cover $\cS_\gamma$, we can find a model whose suboptimality is at most $\cO(\gamma L_z\epsilon)$. 
To contextualize the bound in Proposition \ref{prop:cover}, consider an approach that uses the same cover $\cS_\gamma$ but only relies on zeroth-order feedback, that is, $\{\PR(\theta):\theta\in \cS_\gamma\}$. Then, the only feasible estimate of $\PR$ over the whole space is $\widehat\PR(\theta) = \PR(\Pi_{\cS_\gamma}(\theta))$, where $\Pi_{\cS_\gamma}(\theta) = \argmin_{\theta'\in\cS_\gamma} \|\theta-\theta'\|$ is the projection onto the cover  $\cS_\gamma$. This zeroth-order approach only guarantees an accuracy of $|\PR(\theta)-\PRhat(\theta)|\leq (L_z\epsilon + L_\theta)\gamma$, a strictly weaker approximation than the one in Proposition \ref{prop:cover}.

\subsection{Sequential elimination of suboptimal models}

\begin{figure}[t!]
    \centering
    \subfloat[Baseline confidence bounds]
    {\includegraphics[height=0.21\columnwidth]{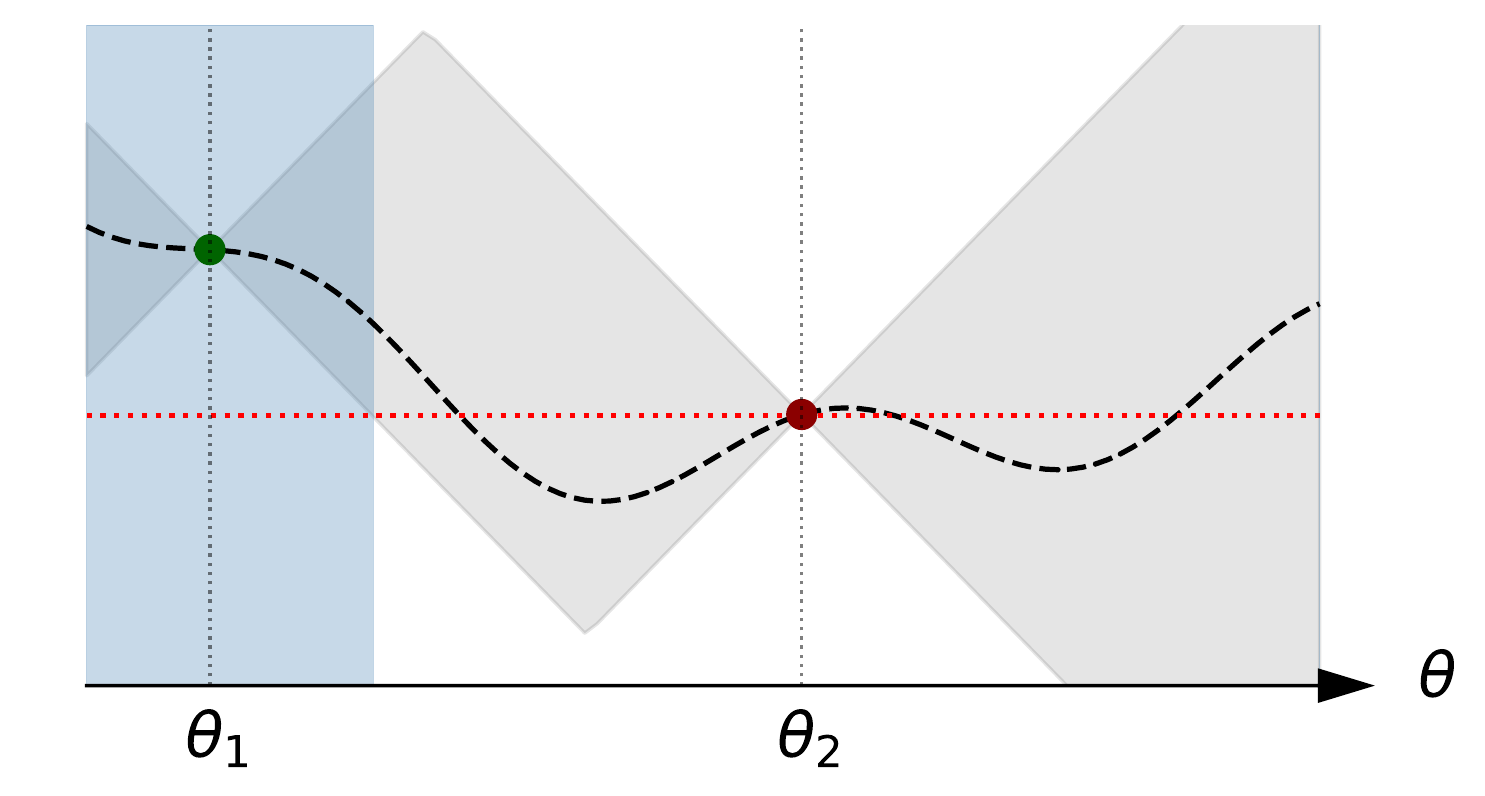}\label{fig:disca}} \hspace{1cm}
    \subfloat[Performative confidence bounds $\quad\quad\quad\quad\quad$]
    {\includegraphics[height=0.21\columnwidth]{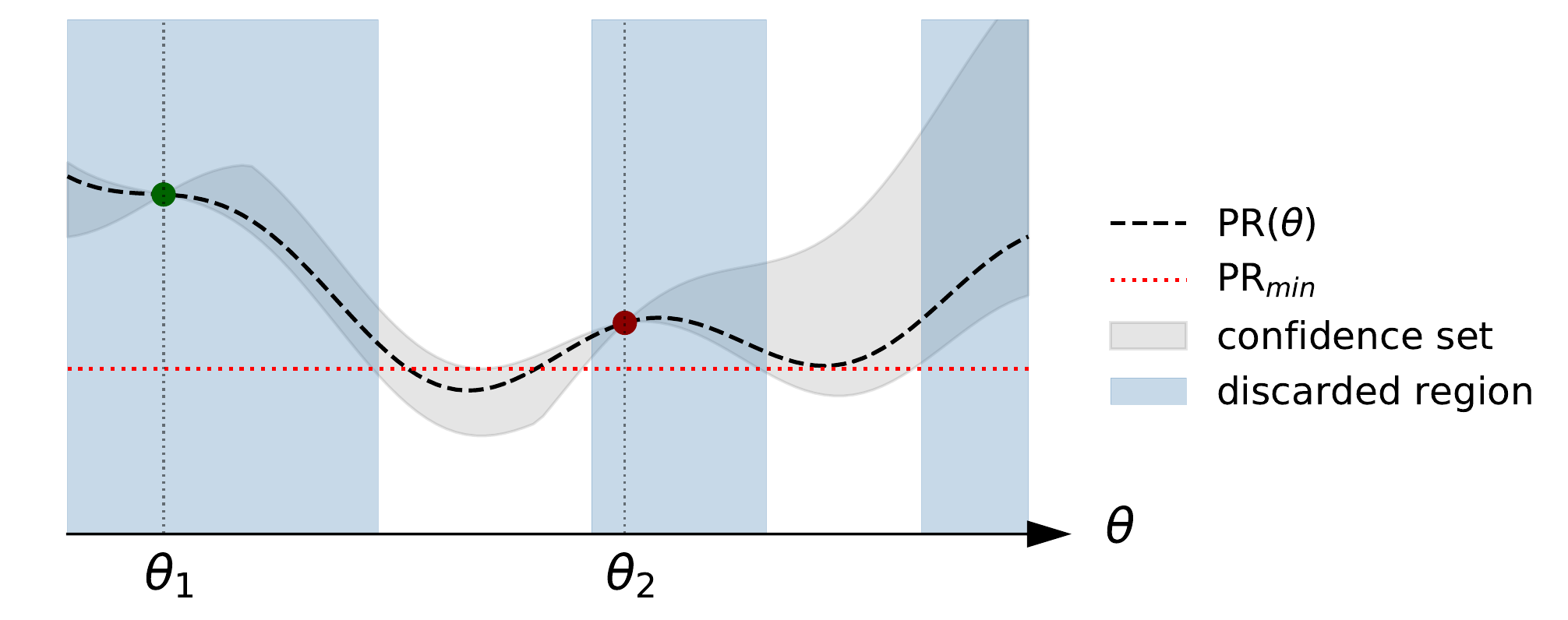}\label{fig:discb}}
    \caption{Performative feedback allows discarding unexplored suboptimal models even in regions that have not been explored. A model $\theta$ is discarded if  $\PR_\text{LB}(\theta)>\PRmin$. 
    The loss function and feedback model are the same as in Figure~\ref{fig:confidence}.}
    \label{fig:compare-discarding}
\end{figure}



Now we show how performative confidence bounds can guide exploration. Specifically, we show that every deployment informs the risk of unexplored models, which allows us to sequentially discard suboptimal regions of the parameter space.

To develop a formal procedure for discarding points, let $\PR_\mathrm{LB}(\theta)$ denote a lower confidence bound on $\PR(\theta)$ and $\PRmin$ denote an upper confidence bound on $\PR(\thetaPO)$ based on the information from the models deployed so far:
\begin{align*}
    &\PR_\mathrm{LB}(\theta) = \max_{\theta' \text{ already deployed}} \left(\DPR(\theta',\theta) - L_z \epsilon \|\theta-\theta'\|\right),\\
    &\PRmin =  \min_{\theta\in\Theta} \;\min_{\theta' \text{ already deployed}}  \;\left(\DPR(\theta',\theta) +L_z\epsilon  \|\theta'-\theta\|\right).
\end{align*}
It is not difficult to see that the following lower bound on the suboptimality of model $\theta$ holds:
\begin{proposition}
\label{prop:elimination}
For all $\theta\in\Theta$, we have $\Delta(\theta)\geq \PRLB(\theta)-\PRmin$.
\end{proposition}
In particular, models $\theta$ with $\PRLB(\theta)>\PRmin$ cannot be optimal. We recall our toy example from Figure~\ref{fig:confidence} and illustrate in Figure~\ref{fig:compare-discarding} the parameter configurations we can discard after the deployment of two models, $\theta_1$ and $\theta_2$. We can see that access to $\DPR$ allows us to discard a large portion of the parameter space, and, in contrast to the baseline black-box approach, it is possible to discard regions of the space that have not been explored.

\section{Performative confidence bounds algorithm}\label{sec:algorithm}

We introduce our main algorithm that builds on the two insights from the previous section. We furthermore provide a rigorous, finite-sample analysis of its guarantees.

\subsection{Algorithm overview}
Our \emph{performative confidence bounds} algorithm, formally stated in Algorithm~\ref{alg:adaptive-exploration}, takes advantage of performative feedback by assessing the risk of unexplored models and thus guiding exploration.
We give an overview of the main steps.

Inspired by the successive elimination algorithm \citep{even2002pac}, the algorithm keeps track of and refines an \emph{active} set of models $\cA\subseteq \Theta$. Roughly speaking, active models are those that are estimated to have low risk and only they are admissible to deploy. 
To deal with finite-sample uncertainty, the algorithm proceeds in phases which progressively refine the precision of the finite-sample risk estimates. More precisely, in phase $p$ the algorithm chooses an error tolerance $\gamma_p$ and deploys a model for $n_p$ steps. In each step $m_0$ samples induced by the deployed model are collected, and $n_p$ is chosen so that the inferred estimates of $\DPR$ are $\gamma_p$-accurate. Formally, if $\theta$ is deployed in phase $p$, we collect an empirical distribution $\widehat \cD(\theta)$ of $n_p m_0$ samples so that $|\widehat\DPR(\theta,\theta') - \DPR(\theta,\theta')| \leq \gamma_p$ for all $\theta'$ with high probability, where \[\widehat\DPR(\theta,\theta'):=\E_{z\sim\widehat \cD(\theta)}\, \ell(z;\theta').\] These estimates of $\DPR$ are used to construct performative confidence bounds and refine $\cA$. 

Each phase begins by constructing a net of the current active set $\cA$. The points in the net are sequentially deployed in the phase, unless they are deemed to be suboptimal based on previous deployments in that phase and are in that case eliminated. During phase $p$, we denote by $\cP_p$ the running set of deployed points and by $\cS_p$ the running set of net points that have not been discarded. We initialize $\cS_p$ to a minimal $r_p$-net of the current set of active points $\cA$, denoted $\mathcal{N}_{r_p}(\cA)$, where $r_p$ is proportional to $\gamma_p$. A net point $\theta$ gets eliminated from $\cS_p$ if no point in $\text{Ball}_{r_p}(\theta) := \{\theta'\in\Theta: \|\theta'-\theta\| \leq r_p\}$ is active. This means that we may deploy suboptimal points in the net if they help inform active points nearby.

\begin{algorithm}[t!]
\caption{Performative Confidence Bounds Algorithm}\label{alg:adaptive-exploration}
\begin{algorithmic}[1]
\Require time horizon $T$, number of samples collected per step $m_0$, sensitivity parameter $\epsilon$, Lipschitz constant $L_z$, complexity bound $\mathfrak{C}$
\State Initialize $\cA \leftarrow \Theta$
\For{phase $p=0,1,\dots$} 
\State Set error tolerance $\gamma_p = 2^{-p}$ and net radius $r_p = \frac{\gamma_p}{L_z \epsilon}$
\State Let $n_p = \left\lceil \frac{\left(2 \mathfrak{C} + 3 \sqrt{\log T}\right)^2}{ \gamma_p^2 m_0}\right\rceil$ 
\State Initialize $\mathcal S_p\leftarrow \mathcal N_{r_p}(\cA)$ \Comment{Initialize $\cS_p$ to minimal $r_p$-cover of $\cA$}
\State Initialize $\cP_p \leftarrow \emptyset$
\While{$\mathcal S_p\neq \emptyset$}
\State Draw $\theta_{\text{net}}\in\mathcal S_p$ uniformly at random
\State Deploy $\theta_t$ for $n_p$ steps to form $\widehat \DPR(\theta_{\text{net}},\cdot)$
\State $\mathcal S_p \leftarrow \mathcal S_p \setminus \theta_{\text{net}}$ 
\State $\cP_p\leftarrow \cP_p \cup \theta_{\text{net}}$ \Comment{Update set of deployed models}
\State $\PRmin \leftarrow \min_{\theta\in\Theta} \min_{\theta' \in \cP_p} \,\widehat{ \mathrm{DPR}}(\theta',\theta)+L_z\epsilon  \|\theta'-\theta\|$ \Comment{Update estimate of $\PR(\thetaPO)$} 
\State $\PR_\text{LB}(\theta)\leftarrow \max_{\theta' \in \cP_p} \left(\widehat{ \mathrm{DPR}}(\theta',\theta)-L_z\epsilon \|\theta'-\theta\|\;\right)\forall \theta\in\cA$\Comment{Update LB for all models}
\State $\mathcal A \leftarrow \mathcal A \setminus \mathcal \{\theta\in\cA: \PR_{\text{LB}}(\theta)> \PRmin + 2 \gamma_p\}$ \Comment{Update active region} 
\State $\mathcal S_p \leftarrow \mathcal S_p \setminus \{\theta \in \mathcal S_p : \text{Ball}_{r_p}(\theta)\cap \cA = \emptyset \}$ \Comment{Remove net points in deactivated regions} 
\EndWhile
\EndFor
\end{algorithmic}
\end{algorithm}

\subsection{Comparison with adaptive zooming algorithm}

While we borrow the idea of an instance-dependent zooming dimension from \citet{kleinberg2008multi}, Algorithm \ref{alg:adaptive-exploration} and its analysis are substantially different from prior work. In particular, \citet{kleinberg2008multi} study an adaptive zooming algorithm which combines a UCB-based approach with an arm activation step. Adapting this method to our setting encounters several obstacles that we describe below. 

First, a naive application of the adaptive zooming algorithm proposed by \citet{kleinberg2008multi} does not lead to sublinear regret in our setting, unless we assume Lipschitzness of $\PR$. Their rule for activating new arms requires that the reward of arms within a given radius in Euclidean distance of the pulled arm is similar. However, without Lipschitzness of $\PR$, there is no radius that would ensure this property. 

Given the shortcomings of this exploration strategy, one might imagine that selecting a better distance between arms, e.g. one based on performative confidence bounds, would result in a better algorithm. A natural distance function would be $d(\theta, \theta')$ taken as (an empirical estimate of) $\PR(\theta) - \DPR(\theta, \theta') + L_z \epsilon \|\theta - \theta'\|$. The challenge is that the analysis in \citep{kleinberg2008multi} explicitly requires symmetry of the distance function, which $d(\theta, \theta')$ violates.

Therefore, to single out the $L_z \epsilon$ dependence, it is necessary to disentangle learning the structure of the distribution map from the elimination of arms based on reward, which is in stark contrast with UCB-style adaptive zooming algorithms. Algorithm \ref{alg:adaptive-exploration} achieves this by relying on a novel adaptation of successive elimination.

\subsection{Regret bound}

Before we state the regret bound for Algorithm~\ref{alg:adaptive-exploration}, let us comment on an important component in the analysis. Recall that throughout the algorithm we operate with finite-sample estimates of the decoupled performative risk to bound the risk of unexplored models. Specifically, for any deployed $\theta$, we make use of $\widehat{\mathrm{DPR}}(\theta,\theta')$ for \emph{all} $\theta'$. Since we need these estimates to be valid simultaneously for all $\theta'$, we rely on uniform convergence. As such, the Rademacher complexity of the loss function class naturally enters the bound.

\begin{definition}[Rademacher complexity]
Given a loss function $\ell(z;\theta)$, we define $\mathfrak{C}^*(\ell)$ to be:
\[\mathfrak{C}^*(\ell) 
= \sup_{\theta \in \Theta} \;\sup_{n \in\mathbb{N}}\; \sqrt{n} \cdot \E_{\epsilon, z^{\theta}} \left(\sup_{\theta' \in \Theta} \Big|\frac{1}{n} \sum_{j=1}^n \epsilon_j \ell(z_j^{\theta}; \theta')\Big|\right),\]
where $\epsilon_j\sim \mathrm{Rademacher}$ and $z_j^{\theta}\sim \cD(\theta)$, $\forall j\in[n]$, which are all independent of each other.
\end{definition}

Now we can state our regret guarantee for Algorithm~\ref{alg:adaptive-exploration}.
\begin{theorem}[Main regret bound]
\label{thm:fine-grained-seq}
Assume the loss $\ell(z;\theta)$ is $L_z$-Lipschitz in $z$ and let $\epsilon$ denote the sensitivity of the distribution map. Suppose that $\mathfrak{C}$ is any value such that $\mathfrak{C}^*(\ell) \leq \mathfrak{C}$ and $m_0 = o(\mathcal{B}_{\log T,\mathfrak{C}}^2)$,
where $\mathcal{B}_{\log T,\mathfrak{C}}:=\sqrt{\log T} + \mathfrak{C}$.
Then, after $T$ time steps, Algorithm~\ref{alg:adaptive-exploration} achieves a regret bound of
\begin{align*}
  \Reg(T) = \cO\Bigg(T^{\frac{d+1}{d+2}}\left(\frac{ (L_z \epsilon)^d \mathcal{B}_{\log T,\mathfrak{C}}^2}{m_0}\right)^{\frac{1}{d+2}}  + \sqrt{T} \;\frac{\mathcal{B}_{\log T,\mathfrak{C}}}{\sqrt{m_0}} \Bigg),
\end{align*}
where $d$ is the $(L_z\epsilon)$-sequential zooming dimension (see Definition \ref{def:zdim-seq-our-algo}).
\end{theorem}
\begin{remark}[Consequences for finding performative optima]
Algorithm~\ref{alg:adaptive-exploration} has the additional property that it generates a model with near-minimal performative risk. In particular, an intermediate step in the proof of Theorem~\ref{thm:fine-grained-seq} shows if $T$ is sufficiently large, the final iterate $\theta_T$ of Algorithm~\ref{alg:adaptive-exploration} satisfies:
\[\E \left[\PR(\theta_T) - \min_{\theta \in \Theta} \PR(\theta)\right] \leq  \cO\left(T^{-\frac{1}{d+2}} \left(\frac{ (L_z \epsilon)^d \mathcal{B}_{\log T,\mathfrak{C}}^2}{m_0}\right)^{\frac{1}{d+2}}  \right), \]
where $d$ is the $(L_z \epsilon)$-zooming dimension. 
\end{remark}

Notice that the regret in Theorem \ref{thm:fine-grained-seq} depends on the sequential zooming dimension (formally defined in Definition \ref{def:zdim-seq-our-algo}). This sequential variant of zooming dimension accounts for the sequential elimination of models within each phase. We will show in the next section that the sequential zooming dimension is upper bounded by the usual zooming dimension (see Proposition~\ref{prop:zooming}).  

The primary advantage of Theorem \ref{thm:fine-grained-seq} over the Lipschitz bandit baseline can be seen by examining the first term in the regret bound. This term resembles the black-box regret bound from Section \ref{sec:blackbox}; however, the key difference is that that the bound of Theorem~\ref{thm:fine-grained-seq} depends on the complexity of the distribution map rather than that of the performative risk. In particular, the Lipschitz constant is $L_z \epsilon$ and not $L_{\theta} + L_z \epsilon$. The advantage is pronounced when $\epsilon \rightarrow 0$, making the first term of the bound in Theorem~\ref{thm:fine-grained-seq} vanish so only the $\cO(\sqrt{T})$ term remains. On the other hand, the bound in Proposition \ref{prop:zooming} maintains an exponential dimension dependence. 

Taking the limit as $\epsilon \rightarrow 0$ also reveals why the second term in the bound emerges. Even if the distribution map is constant, there is regret arising from finite-sample error. This is a key conceptual difference in the meaning of Lipschitzness of the distribution map versus that of the performative risk: $L_{\theta} + L_z \epsilon$ being $0$ implies that $\PR$ is flat and thus all models are optimal, while performative regret minimization is nontrivial even if $L_z \epsilon = 0$. Unlike the first term, the second term due to finite samples is dimension-independent apart from any dependence implicit in the Rademacher complexity.

We note that the presence of the Rademacher complexity term $\mathfrak{C}^*(\ell)$ makes a direct comparison of the bound in Theorem \ref{thm:fine-grained-seq} and the bound in Proposition \ref{prop:zooming} subtle. When the Rademacher complexity is very high, the regret bound in Theorem \ref{thm:fine-grained-seq} may be worse. Nonetheless, for many natural function classes, the Rademacher complexity is polynomial in the dimension; in these cases, Theorem \ref{thm:fine-grained-seq} can substantially outperform the regret bound in Proposition \ref{prop:zooming}.

Another key feature of the regret bound in Theorem \ref{thm:fine-grained-seq} worth highlighting is the zooming dimension. Definition~\ref{def:zdim} allows us to directly compare the dimension in Theorem~\ref{thm:fine-grained-seq} with the dimension in Proposition \ref{prop:zooming}:
the $(L_z\epsilon)$-zooming dimension of Algorithm~\ref{alg:adaptive-exploration} is no larger than, and most likely smaller than, the $(L_{\theta} + L_z\epsilon)$-zooming dimension in the black-box approach. Moreover, the sequential variant of zooming dimension in Theorem~\ref{thm:fine-grained-seq} can further reduce the dimension. 

Finally, the main assumption underpinning the bound in Theorem \ref{thm:fine-grained-seq} is that $\DPR$ is $(L_z \epsilon)$-Lipschitz in its first argument. Assumption~\ref{ass:lipschitz} coupled with Lipschitzness of the loss in the data achieves this. However, this property can hold with different regularity assumptions on the distribution map and loss function; e.g., if the loss is bounded and the distribution map is Lipschitz in total variation distance.

\begin{figure}[t!]
    \centering
    \subfloat
    {\includegraphics[height=0.21\columnwidth]{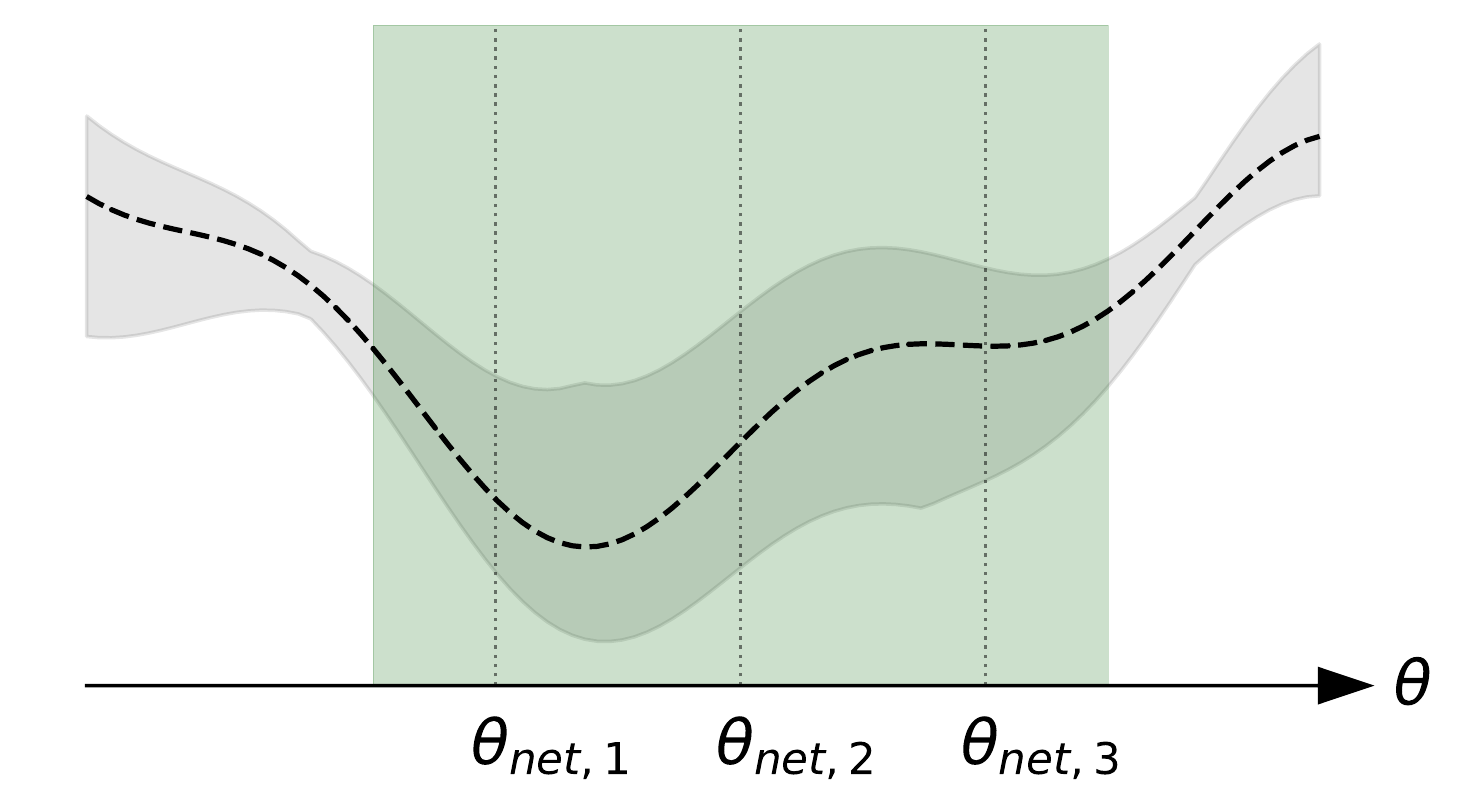}\label{fig:disc--a}}\hspace{1cm}
    \subfloat
    {\includegraphics[height=0.21\columnwidth]{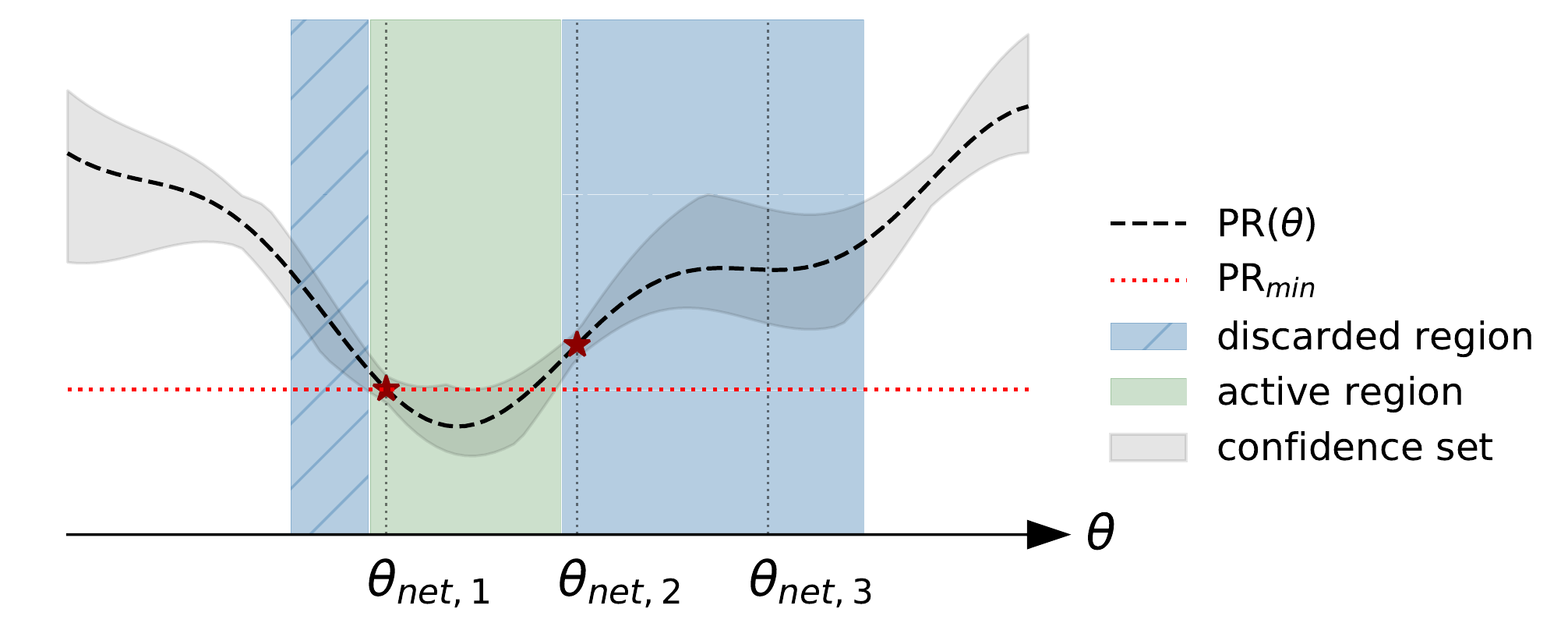}\label{fig:disc--b}}
    \caption{Sequential deployment of models allows Algorithm~\ref{alg:adaptive-exploration} to eliminate points from $\cS_p$, reducing the number of deployments during the phase. We see how the deployment of $\theta_{\text{net},1}$ and $\theta_{\text{net},2}$ allows one to eliminate $\theta_{\text{net},3}$.}
    \label{fig:sequential}
\end{figure}



\subsection{Sequential zooming dimension}

The zooming dimension of Definition~\ref{def:zdim} does not take into account that, using performative feedback, our algorithm can eliminate unexplored models \textit{within} a phase. We illustrate the benefits of this sequential exploration strategy in Figure \ref{fig:sequential}, where the deployment of two models is sufficient to eliminate the remaining model in the cover. This motivates a sequential definition of zooming dimension that captures the benefits of sequential exploration. 

To set up the definition of \textit{sequential zooming dimension}, we need to introduce some notation. For a set of points $\mathcal{S}$, enumeration $\pi:\mathcal{S}\rightarrow \{1,\dots,|\mathcal{S}|\}$ that specifies an ordering on $\cS$, and number  $k\in\{1,\dots,|\cS|\}$, let
\begin{align*}
\PR_{\mathrm{LB}}(\theta;k) &:= \max_{\theta'\in\cS:\pi(\theta')< k} \left(\DPR(\theta',\theta) - L_z\epsilon \|\theta-\theta'\|\right),\\
\PR^s_{\mathrm{LB}}(k) &:= \min_{\theta\in \mathrm{Ball}_s(\pi^{-1}(k))} \PR_{\mathrm{LB}}(\theta; k),\\
\PRmin(k) &:= \min_{\theta}  \min_{\theta'\in\cS:\pi(\theta')< k}  \left(\DPR(\theta',\theta) + L_z\epsilon \|\theta'-\theta\|\right).
\end{align*}
\noindent Here, $\PR_{\mathrm{LB}}(\theta; k)$ is a lower bound on $\PR(\theta)$ arising from the first $k-1$ deployments of the phase. Similarly, $\PR^s_{\mathrm{LB}}(k)$ captures the minimal lower confidence bound on the performative risk for any point in an $s$-ball around the $k$-th deployed model, $\pi^{-1}(k)$.
Finally, $\PRmin(k)$ captures an upper bound on $\PR(\thetaPO)$, estimated from the first $k-1$ deployments.

Using the above terms, we see that $\PR^s_{\mathrm{LB}}(k) \leq \PRmin(k) + 4\alpha s$ is the population version of the condition that a model in the cover does not get discarded. The sequential zooming dimension captures the maximal number of models in each suboptimality band that can be deployed. 

\begin{definition}[Sequential zooming dimension]
\label{def:zdim-seq-our-algo}
A performative prediction problem instance has $\alpha$-zooming dimension equal to $d$ if for any minimal $s$-cover $\mathcal{S}$ of any subset of $\{\theta:\Delta(\theta)\leq 16 \alpha s\}$ and all $0<r\leq s\leq 1$, the expected number of models $\theta\in\mathcal{S} \cap \{\theta: 16 \alpha r \leq \Delta(\theta) < 32 \alpha r\}$ with
\begin{equation}
\label{eq:seq_zoom}
  \PR^s_{\mathrm{LB}}(\pi(\theta)) \leq \PRmin(\pi(\theta)) + 4\alpha s  
\end{equation}
is at most a constant multiple of $(3/s)^{d}$, where the expectation is taken over a uniformly sampled enumeration $\pi:\cS\rightarrow \{1,\dots,|\mathcal{S}|\}$. 
\end{definition}

The sequential zooming dimension is bounded by the zooming dimension in Definition \ref{def:zdim}.
\begin{proposition}
\label{prop:zooming}
For all $\alpha>0$, the $\alpha$-zooming dimension is at least as large as the $\alpha$-sequential zooming dimension.
\end{proposition}
\noindent The claim of Proposition \ref{prop:zooming} follows by definition. To see this, let $d$ be the $\alpha$-zooming dimension. This means that $\mathcal{S}$ includes at most a constant multiple of $(3/s)^{d}$ elements from $\{\theta: 16 \alpha r \le \Delta(\theta)< 32 \alpha r\}$, for all $0<r\leq s\leq 1$. This immediately guarantees that the subset of $\cS$ characterized by \eqref{eq:seq_zoom} is at most a multiple of $(3/s)^{d}$, as desired. 

In Appendix \ref{appendix:seqzooming}, we provide an example where the sequential zooming dimension is \textit{strictly} smaller than the zooming dimension.


\section{Regret minimization for location families}

\begin{algorithm}[b!]
\caption{Performative Regret Minimization for Location Families}\label{alg:location-families}
\begin{algorithmic}[1]
\Require time horizon $T$, number of samples collected per step $m_0$, base distribution $\cD_0$, bound $M_*$ such that $\|\mu_*\| \leq M_*$
\State Initialize confidence set $\mathcal{C}_1 \leftarrow \{\mu: \|\mu\| \leq M_*\}$
\For{step $t=1,2,\dots$} 
\State $\PRLB(\theta) \leftarrow \min_{\mu\in\mathcal{C}_t} \E_{z_0\sim \cD_0} \ell(z_0 + \mu\theta;\theta)~\forall \theta\in\Theta$\Comment{Update LB for all models}
\State Deploy $\theta_t = \argmin_\theta \PRLB(\theta)$ \Comment{Deploy model with lowest LB}
\State Compute $\bar z_t = \frac{1}{m_0} \sum_{i=1}^{m_0} z_t^{(i)}$ from collected samples
\State Let $\Sigma_t \leftarrow  \sum_{i=1}^t \theta_i\theta_i^\top + \frac{1}{m_0} I$
\State  $\hat \mu_t \leftarrow \Sigma_t^{-1} \left(\sum_{i=1}^t \theta_i \bar z_i^\top\right)$
\Comment{Update estimate of $\mu_*$}
\State $\mathcal{C}_{t+1} \leftarrow \left\{\mu: \left\|\Sigma_t^{1/2}(\hat \mu_{t} - \mu)\right\|  <  \frac{M_* + \sqrt{8m_0 + 8\log T + 2d_\Theta\log\left(1+\frac{T m_0}{d_\Theta}\right)}}{\sqrt{m_0}}\right\}$\Comment{Update confidence set}
\EndFor
\end{algorithmic}
\end{algorithm}

In this section, we show how further knowledge about the structure of the distribution map can help reduce the complexity of performative regret minimization, without necessarily implying favorable structure of the performative risk. Once again, we apply our guiding principle of focusing exploration on learning the distribution map. Since the loss function is known, we can extrapolate knowledge about the distribution map to estimate the performative risk.

We focus on the setting of \textit{location families}  \citep{miller21echo}, which are distribution maps that depend on $\theta$ via a linear shift. More precisely, location families are distribution maps of the form $z \sim \cD(\theta) \Leftrightarrow z \stackrel{d}{=} z_0 + \mu_*^\top \theta$, 
where  $\mu_*\in\mathbb{R}^{d_\Theta \times m}$ is an unknown matrix and $z_0\in\mathbb{R}^m$ is a zero-mean subgaussian sample from a base distribution $\cD_0$. 
\begin{example}[Strategic classification]
Location families arise in strategic classification~\citep{hardt16strat}, where agents strategically manipulate their features in response to a deployed model. Suppose the learner uses a linear predictor $f_\theta(x)=\theta^T x$ and the agents incur quadratic cost for changing their original features $x$ to manipulated features $x'$, $c(x,x')=\frac 1 2 (x-x')\Lambda(x-x')$. Then, the best response of an agent, typically modeled as $x_{\mathrm{BR}}(\theta)=\argmax_{x'} f_\theta(x')-c(x,x')$, satisfies the location family structural assumption with $z_0$ being the agent's original features and $\mu^*=\Lambda^{-1}$. 
\end{example}

At a high level, our algorithm can be described as follows: at every step $t$, the learner deploys a model $\theta_t$ and collects $m_0$ samples from $\cD(\theta_t)$. 
We will write $\bar z_t := \frac{1}{m_0} \sum_{i=1}^{m_0} z_t^{(i)}$ for the corresponding sample average at time $t$.
Then, based on all samples collected so far, the algorithm computes the least-squares estimate of $\mu_*$ along with a confidence region for $\mu_*$. In the next step the algorithm picks the model that minimizes a lower confidence bound $\PRLB(\theta)$. See Algorithm~\ref{alg:location-families} for details.

This algorithm is inspired by LinUCB \citep{li2010contextual}, a standard bandits algorithm for linear rewards whose regret scales as $\tilde \cO(d \sqrt{T})$, where $d$ is the dimension of the linear map. Importantly, unlike in the LinUCB analysis, our objective function $\PR(\theta)$ is \emph{not} linear in $\theta$. Still, the nature of performative feedback allows us to learn the hidden linear structure in the distribution map and apply this knowledge to obtain confidence bounds on the performative risk. Below we state our algorithm for performative regret minimization for location families together with its regret guarantees.

\begin{theorem}
\label{thm:lin_bound}
Suppose that $\ell(z;\theta)$ is $L_z$-Lipschitz in $z$, $\cD_0$ is $1$-subgaussian, and $m_0 = o(\log T)$. Then, after $T$ time steps, Algorithm~\ref{alg:location-families} achieves a regret bound of
\[\Reg(T) = \tilde \cO  \left(\frac 1 {\sqrt{m_0}} \max\{L_z, 1\} \sqrt{T}\max\left\{d_\Theta,\sqrt{d_\Theta m}\right\}\right).\]
\end{theorem}

\begin{remark}
\label{rem:UC}
For simplicity, we assume that $\cD_0$ is known in Algorithm \ref{alg:location-families}. This assumption is justified, for example, when we have plenty of historical data about a population, before any model deployment. We note that Theorem \ref{thm:lin_bound} can be extended to the case where we only have a finite data set from $\cD_0$, by relying on a uniform convergence argument. 
\end{remark}

Theorem \ref{thm:lin_bound} shows that by leveraging the hidden linear structure of the distribution map, Algorithm \ref{alg:location-families} inherits the $\tilde \cO(\sqrt{T})$ rate of LinUCB. This bears resemblance to the regret bound in Theorem \ref{thm:fine-grained-seq} that also scaled primarily with the complexity of the distribution map. Furthermore, similarly to Algorithm \ref{alg:adaptive-exploration}, we see that the regret bound for Algorithm~\ref{alg:location-families} holds while allowing the loss to have arbitrary dependence on $\theta$. For example, the loss need not be convex and, as a result, the performative risk need not be convex either. 

We conclude by comparing Theorem \ref{thm:lin_bound} to \citep{miller21echo}, which provided an algorithm for finding performative optima for location families in the special case when the performative risk is \textit{strongly convex}.  Converting their optimization error into a regret bound yields a bound of $\cO(\sqrt{T}(d_\Theta+m))$. While this bears resemblance to Theorem \ref{thm:lin_bound}, the rates are not directly comparable. The algorithm by \citet{miller21echo} does not assume knowledge of the base distribution $\cD_0$, but rather deploys the model $\theta=0$ in initial steps to collect samples from $\cD_0$ (see Remark \ref{rem:UC} for how to combine this strategy with our algorithm).
In any case, the main benefit of Theorem~\ref{thm:lin_bound} is that it applies to a more general setting, placing significantly fewer restrictions on the loss function and the performative risk.

\section{Future directions}

Having illuminated the connection between performative prediction and bandit problems, our work opens the door for interesting further investigations. We highlight several directions we consider promising.

\paragraph{Structural knowledge of the distribution map.} Domain knowledge about performative distribution shifts is sometimes available: for example, a parametric approximation to the aggregate response~\citep{miller21echo, izzo2021learn}, a microfoundations model for individual behavior~\citep{hardt16strat, jag21alt}, or basic constraints on the agents' action set~\citep{chen2020learning}.
For linear shifts, we demonstrated how such structural knowledge about the distribution map can help guide exploration. We expect this principle to apply to other structures of $\cD(\theta)$.

\paragraph{Consequences of exploration.} 
An important limitation of exploration in performative environments are social welfare concerns. Performative shifts can rarely be analyzed offline and every model deployment is consequential for the population the model acts upon. 
The ability to discard highly suboptimal regions of the parameter space without having to deploy a model within is highly appealing from a welfare perspective.
Beyond this, we believe that incorporating constraints on what constitutes safe exploration~\citep{wu2016conservative,turchetta19safe,kazerouni19bandit} is crucial for performative optimization in practice.

\paragraph{Costs of a new deployment.} 
Our notion of regret quantifies the statistical complexity of regret minimization, but it does not differentiate between collecting more samples induced by the currently deployed model and deploying a new model. This difference has previously been studied by \citet{mendler20stochasticPP} in the context of stochastic retraining methods. 
Due to the costs associated with a new deployment, collecting more samples from the same model typically comes at a reduced cost for the learner, and there may be a better notion of regret that reflects this.

\paragraph{Adapting to unknown sensitivity.} Our algorithm relies on knowing $L_z \epsilon$.
While the Lipschitzness of a classifier in the data has been studied in the context of adversarial robustness~\citep{szegedy14spectral,cisse17adv,hein17localLip,yang20rob}, which could help inform $L_z$, the sensitivity $\epsilon$ of an environment is generally unknown.
Adapting the tools by~\citet{bubeck11woL} could help relax the requirement of a known sensitivity.


\paragraph{``Best of both worlds'' algorithm.} When the Rademacher complexity of the function class is high, the Lipschitz bandit baseline may provide a better regret bound than Algorithm~\ref{alg:adaptive-exploration}. It would be an interesting task for future work to design an algorithm that intersects the confidence sets of both algorithms and inherits the better of the two regret bounds.

\section*{Acknowledgements}

The authors would like to thank Moritz Hardt for helpful conversations during the course of this project, Nilesh Tripuraneni for pointers to relevant literature, 
and Jacob Steinhardt, Alex Wei, and Clara Wong-Fannjiang for valuable feedback on the manuscript.

\bibliography{references}

\begin{thebibliography}{40}
\providecommand{\natexlab}[1]{#1}
\providecommand{\url}[1]{\texttt{#1}}
\expandafter\ifx\csname urlstyle\endcsname\relax
  \providecommand{\doi}[1]{doi: #1}\else
  \providecommand{\doi}{doi: \begingroup \urlstyle{rm}\Url}\fi

\bibitem[Agrawal(1995)]{agrawal1995continuum}
Rajeev Agrawal.
\newblock The continuum-armed bandit problem.
\newblock \emph{SIAM Journal on Control and Optimization}, 33\penalty0
  (6):\penalty0 1926--1951, 1995.

\bibitem[Auer et~al.(2007)Auer, Ortner, and Szepesv{\'a}ri]{auer2007improved}
Peter Auer, Ronald Ortner, and Csaba Szepesv{\'a}ri.
\newblock Improved rates for the stochastic continuum-armed bandit problem.
\newblock In \emph{International Conference on Computational Learning Theory},
  pages 454--468. Springer, 2007.

\bibitem[Balcan et~al.(2015)Balcan, Blum, Haghtalab, and
  Procaccia]{balcan2015commitment}
Maria-Florina Balcan, Avrim Blum, Nika Haghtalab, and Ariel~D Procaccia.
\newblock Commitment without regrets: Online learning in {S}tackelberg security
  games.
\newblock In \emph{Proceedings of the 16th ACM Conference on Economics and
  Computation}, pages 61--78, 2015.

\bibitem[Bechavod et~al.(2021)Bechavod, Ligett, Wu, and
  Ziani]{bechavod2021gaming}
Yahav Bechavod, Katrina Ligett, Steven Wu, and Juba Ziani.
\newblock Gaming helps! learning from strategic interactions in natural
  dynamics.
\newblock In \emph{International Conference on Artificial Intelligence and
  Statistics}, pages 1234--1242, 2021.

\bibitem[Brown et~al.(2022)Brown, Hod, and Kalemaj]{brown2020performative}
Gavin Brown, Shlomi Hod, and Iden Kalemaj.
\newblock Performative prediction in a stateful world.
\newblock In \emph{International Conference on Artificial Intelligence and
  Statistics}, pages 6045--6061, 2022.

\bibitem[Bubeck et~al.(2011)Bubeck, Stoltz, and Yu]{bubeck11woL}
S{\'e}bastien Bubeck, Gilles Stoltz, and Jia~Yuan Yu.
\newblock Lipschitz bandits without the {L}ipschitz constant.
\newblock In \emph{International Conference on Algorithmic Learning Theory},
  pages 144--158, 2011.

\bibitem[Chen et~al.(2020)Chen, Liu, and Podimata]{chen2020learning}
Yiling Chen, Yang Liu, and Chara Podimata.
\newblock Learning strategy-aware linear classifiers.
\newblock \emph{Advances in Neural Information Processing Systems},
  33:\penalty0 15265--15276, 2020.

\bibitem[Cisse et~al.(2017)Cisse, Bojanowski, Grave, Dauphin, and
  Usunier]{cisse17adv}
Moustapha Cisse, Piotr Bojanowski, Edouard Grave, Yann Dauphin, and Nicolas
  Usunier.
\newblock Parseval networks: Improving robustness to adversarial examples.
\newblock In \emph{Proceedings of the 34th International Conference on Machine
  Learning}, page 854–863, 2017.

\bibitem[Cohen et~al.(2016)Cohen, Hazan, and Koren]{alon16graph}
Alon Cohen, Tamir Hazan, and Tomer Koren.
\newblock Online learning with feedback graphs without the graphs.
\newblock In \emph{International Conference on Machine Learning}, pages
  811--819, 2016.

\bibitem[Dong et~al.(2018)Dong, Roth, Schutzman, Waggoner, and
  Wu]{dong2018strategic}
Jinshuo Dong, Aaron Roth, Zachary Schutzman, Bo~Waggoner, and Zhiwei~Steven Wu.
\newblock Strategic classification from revealed preferences.
\newblock In \emph{Proceedings of the 2018 ACM Conference on Economics and
  Computation}, pages 55--70, 2018.

\bibitem[Dong and Ratliff(2021)]{dong2021approximate}
Roy Dong and Lillian~J Ratliff.
\newblock Approximate regions of attraction in learning with decision-dependent
  distributions.
\newblock \emph{arXiv preprint arXiv:2107.00055}, 2021.

\bibitem[Drusvyatskiy and Xiao(2020)]{drusvyatskiy2020stochastic}
Dmitriy Drusvyatskiy and Lin Xiao.
\newblock Stochastic optimization with decision-dependent distributions.
\newblock \emph{arXiv preprint arXiv:2011.11173}, 2020.

\bibitem[Even-Dar et~al.(2002)Even-Dar, Mannor, and Mansour]{even2002pac}
Eyal Even-Dar, Shie Mannor, and Yishay Mansour.
\newblock {PAC} bounds for multi-armed bandit and {M}arkov decision processes.
\newblock In \emph{International Conference on Computational Learning Theory},
  pages 255--270, 2002.

\bibitem[Fiez and Ratliff(2020)]{fiez2020local}
Tanner Fiez and Lillian~J Ratliff.
\newblock Local convergence analysis of gradient descent ascent with finite
  timescale separation.
\newblock In \emph{International Conference on Learning Representations}, 2020.

\bibitem[Fiez et~al.(2020)Fiez, Chasnov, and Ratliff]{fiez2020implicit}
Tanner Fiez, Benjamin Chasnov, and Lillian Ratliff.
\newblock Implicit learning dynamics in {S}tackelberg games: Equilibria
  characterization, convergence analysis, and empirical study.
\newblock In \emph{International Conference on Machine Learning}, pages
  3133--3144, 2020.

\bibitem[Hardt et~al.(2016)Hardt, Megiddo, Papadimitriou, and
  Wootters]{hardt16strat}
Moritz Hardt, Nimrod Megiddo, Christos Papadimitriou, and Mary Wootters.
\newblock Strategic classification.
\newblock In \emph{Proceedings of the 2016 ACM Conference on Innovations in
  Theoretical Computer Science}, pages 111--122, 2016.

\bibitem[Hein and Andriushchenko(2017)]{hein17localLip}
Matthias Hein and Maksym Andriushchenko.
\newblock Formal guarantees on the robustness of a classifier against
  adversarial manipulation.
\newblock In \emph{Proceedings of the 31st International Conference on Neural
  Information Processing Systems}, page 2263–2273, 2017.

\bibitem[Izzo et~al.(2021)Izzo, Ying, and Zou]{izzo2021learn}
Zachary Izzo, Lexing Ying, and James Zou.
\newblock How to learn when data reacts to your model: Performative gradient
  descent.
\newblock In \emph{Proceedings of the 38th International Conference on Machine
  Learning}, volume 139, pages 4641--4650, 2021.

\bibitem[Jagadeesan et~al.(2021)Jagadeesan, Mendler{-}D{\"{u}}nner, and
  Hardt]{jag21alt}
Meena Jagadeesan, Celestine Mendler{-}D{\"{u}}nner, and Moritz Hardt.
\newblock Alternative microfoundations for strategic classification.
\newblock In \emph{Proceedings of the 38th International Conference on Machine
  Learning}, volume 139, pages 4687--4697. {PMLR}, 2021.

\bibitem[Jin et~al.(2020)Jin, Netrapalli, and Jordan]{jin2020local}
Chi Jin, Praneeth Netrapalli, and Michael Jordan.
\newblock What is local optimality in nonconvex-nonconcave minimax
  optimization?
\newblock In \emph{International Conference on Machine Learning}, pages
  4880--4889, 2020.

\bibitem[Kazerouni et~al.(2017)Kazerouni, Ghavamzadeh, Abbasi~Yadkori, and
  Van~Roy]{kazerouni19bandit}
Abbas Kazerouni, Mohammad Ghavamzadeh, Yasin Abbasi~Yadkori, and Benjamin
  Van~Roy.
\newblock Conservative contextual linear bandits.
\newblock \emph{Advances in Neural Information Processing Systems}, 30, 2017.

\bibitem[Kleinberg(2004)]{kleinberg2004nearly}
Robert Kleinberg.
\newblock Nearly tight bounds for the continuum-armed bandit problem.
\newblock \emph{Advances in Neural Information Processing Systems},
  17:\penalty0 697--704, 2004.

\bibitem[Kleinberg et~al.(2008)Kleinberg, Slivkins, and
  Upfal]{kleinberg2008multi}
Robert Kleinberg, Aleksandrs Slivkins, and Eli Upfal.
\newblock Multi-armed bandits in metric spaces.
\newblock In \emph{Proceedings of the fortieth annual ACM symposium on Theory
  of computing}, pages 681--690, 2008.

\bibitem[Koc\'{a}k et~al.(2014)Koc\'{a}k, Neu, Valko, and
  Munos]{kocak14sideobs}
Tom\'{a}\v{s} Koc\'{a}k, Gergely Neu, Michal Valko, and Remi Munos.
\newblock Efficient learning by implicit exploration in bandit problems with
  side observations.
\newblock In \emph{Advances in Neural Information Processing Systems},
  volume~27, 2014.

\bibitem[Lattimore and Szepesv{\'a}ri(2020)]{lattimore2020bandit}
Tor Lattimore and Csaba Szepesv{\'a}ri.
\newblock \emph{Bandit algorithms}.
\newblock Cambridge University Press, 2020.

\bibitem[Li et~al.(2010)Li, Chu, Langford, and Schapire]{li2010contextual}
Lihong Li, Wei Chu, John Langford, and Robert~E Schapire.
\newblock A contextual-bandit approach to personalized news article
  recommendation.
\newblock In \emph{Proceedings of the 19th International Conference on World
  Wide Web}, pages 661--670, 2010.

\bibitem[Li and Wai(2022)]{li2021state}
Qiang Li and Hoi-To Wai.
\newblock State dependent performative prediction with stochastic
  approximation.
\newblock In \emph{International Conference on Artificial Intelligence and
  Statistics}, pages 3164--3186, 2022.

\bibitem[Maheshwari et~al.(2022)Maheshwari, Chiu, Mazumdar, Sastry, and
  Ratliff]{maheshwari2021zeroth}
Chinmay Maheshwari, Chih-Yuan Chiu, Eric Mazumdar, Shankar Sastry, and Lillian
  Ratliff.
\newblock Zeroth-order methods for convex-concave min-max problems:
  Applications to decision-dependent risk minimization.
\newblock In \emph{International Conference on Artificial Intelligence and
  Statistics}, pages 6702--6734, 2022.

\bibitem[Mannor and Shamir(2011)]{mannor11graphfeedback}
Shie Mannor and Ohad Shamir.
\newblock From bandits to experts: On the value of side-observations.
\newblock In \emph{Proceedings of the 24th International Conference on Neural
  Information Processing Systems}, page 684–692, 2011.

\bibitem[Mendler-D\"{u}nner et~al.(2020)Mendler-D\"{u}nner, Perdomo, Zrnic, and
  Hardt]{mendler20stochasticPP}
Celestine Mendler-D\"{u}nner, Juan Perdomo, Tijana Zrnic, and Moritz Hardt.
\newblock Stochastic optimization for performative prediction.
\newblock In \emph{Advances in Neural Information Processing Systems},
  volume~33, pages 4929--4939, 2020.

\bibitem[Miller et~al.(2021)Miller, Perdomo, and Zrnic]{miller21echo}
John~P Miller, Juan~C Perdomo, and Tijana Zrnic.
\newblock Outside the echo chamber: Optimizing the performative risk.
\newblock In \emph{Proceedings of the 38th International Conference on Machine
  Learning}, volume 139, pages 7710--7720, 2021.

\bibitem[Perdomo et~al.(2020)Perdomo, Zrnic, Mendler-D{\"u}nner, and
  Hardt]{perdomo20pp}
Juan Perdomo, Tijana Zrnic, Celestine Mendler-D{\"u}nner, and Moritz Hardt.
\newblock Performative prediction.
\newblock In \emph{Proceedings of the 37th International Conference on Machine
  Learning}, volume 119, pages 7599--7609, 2020.

\bibitem[Podimata and Slivkins(2021)]{podimata20lipschitzbandit}
Chara Podimata and Alex Slivkins.
\newblock Adaptive discretization for adversarial {L}ipschitz bandits.
\newblock In \emph{Conference on Learning Theory}, pages 3788--3805, 2021.

\bibitem[Ray et~al.(2022)Ray, Drusvyatskiy, Fazel, and Ratliff]{raydecision}
Mitas Ray, Dmitriy Drusvyatskiy, Maryam Fazel, and Lillian~J Ratliff.
\newblock Decision-dependent risk minimization in geometrically decaying
  dynamic environments.
\newblock In \emph{Proceedings of the Association for the Advancement of
  Artificial Intelligence Conference on AI (AAAI)}, 2022.

\bibitem[Szegedy et~al.(2014)Szegedy, Zaremba, Sutskever, Bruna, Erhan,
  Goodfellow, and Fergus]{szegedy14spectral}
Christian Szegedy, Wojciech Zaremba, Ilya Sutskever, Joan Bruna, Dumitru Erhan,
  Ian~J. Goodfellow, and Rob Fergus.
\newblock Intriguing properties of neural networks.
\newblock In \emph{2nd International Conference on Learning Representations},
  2014.

\bibitem[Turchetta et~al.(2019)Turchetta, Berkenkamp, and
  Krause]{turchetta19safe}
Matteo Turchetta, Felix Berkenkamp, and Andreas Krause.
\newblock Safe exploration for interactive machine learning.
\newblock In \emph{Advances in Neural Information Processing Systems},
  volume~32, 2019.

\bibitem[Wu et~al.(2015)Wu, Gy\"{o}rgy, and Szepesv\'{a}ri]{wu15sideinfo}
Yifan Wu, Andr\'{a}s Gy\"{o}rgy, and Csaba Szepesv\'{a}ri.
\newblock Online learning with {G}aussian payoffs and side observations.
\newblock In \emph{Proceedings of the 28th International Conference on Neural
  Information Processing Systems}, page 1360–1368, 2015.

\bibitem[Wu et~al.(2016)Wu, Shariff, Lattimore, and
  Szepesv{\'a}ri]{wu2016conservative}
Yifan Wu, Roshan Shariff, Tor Lattimore, and Csaba Szepesv{\'a}ri.
\newblock Conservative bandits.
\newblock In \emph{International Conference on Machine Learning}, pages
  1254--1262, 2016.

\bibitem[Yang et~al.(2020)Yang, Rashtchian, Zhang, Salakhutdinov, and
  Chaudhuri]{yang20rob}
Yao-Yuan Yang, Cyrus Rashtchian, Hongyang Zhang, Russ~R Salakhutdinov, and
  Kamalika Chaudhuri.
\newblock A closer look at accuracy vs. robustness.
\newblock In \emph{Advances in Neural Information Processing Systems},
  volume~33, pages 8588--8601, 2020.

\bibitem[Zrnic et~al.(2021)Zrnic, Mazumdar, Sastry, and Jordan]{zrnic2021leads}
Tijana Zrnic, Eric Mazumdar, Shankar Sastry, and Michael Jordan.
\newblock Who leads and who follows in strategic classification?
\newblock \emph{Advances in Neural Information Processing Systems},
  34:\penalty0 15257--15269, 2021.

\end{thebibliography}
\bibliographystyle{plainnat}
\newpage 

\appendix

\onecolumn

\section{Proofs from Section \ref{sec:blackbox} and Section \ref{sec:perf_feedback}}

\subsection{Proof of Lemma \ref{lemma:lipschitz}}

Notice that $\PR(\theta) - \PR(\theta') 
    = \left(\DPR(\theta, \theta) - \DPR(\theta, \theta')\right) + \left(\DPR(\theta, \theta') - \DPR(\theta', \theta')\right)$.
We bound the first difference using Lipschitzness of $\ell$ in $\theta$ as $|\DPR(\theta, \theta) - \DPR(\theta, \theta')| = |\mathrm E_{z\sim\cD(\theta)} [\ell(z;\theta) - \ell(z;\theta')]| \leq  L_\theta \|\theta-\theta'\|$. For the second term we combine Assumption \ref{ass:lipschitz} and Lipschitzness of $\ell$ in $z$ via the Kantorovich-Rubinstein duality theorem. In particular, we get $|\DPR(\theta, \theta') - \DPR(\theta', \theta')| = |\mathrm E_{z\sim\cD(\theta)} \ell(z;\theta') - \mathrm E_{z\sim\cD(\theta')} \ell(z;\theta')| \leq  \epsilon L_z \|\theta-\theta'\|.$
Putting both bounds together, we obtain the claimed Lipschitz bound.

\subsection{Proof of Proposition \ref{prop:cover}}

We construct a $\gamma$-cover of the parameter space, denoted $\cS_\gamma$, and deploy all models in this cover. This gives us access to the distributions $\{\cD(\theta): \theta\in \cS_{\gamma}\}$. Using this information, for any $\theta\in\Theta$ we can compute
\[\widehat \PR(\theta) = \DPR(\Pi_{\cS_\gamma}(\theta), \theta) = \E_{z\sim \cD(\Pi_{\cS_\gamma}(\theta))} \ell(z;\theta),\]
where $\Pi_{\cS_\gamma}(\theta) := \argmin_{\theta'\in\cS_\gamma} \|\theta'-\theta\|$ is the projection onto $\cS_\gamma$. Note that $\|\theta-\Pi_{\cS_\gamma}(\theta)\|\leq \gamma$ all $\theta\in\Theta$ since $\cS_\gamma$ is a cover. Therefore, for any $\theta\in\Theta$, we can bound $\PR(\theta)$ as
\begin{align*}
  \PR(\theta) &\leq \DPR(\Pi_{S_\gamma}(\theta),\theta) + L_z\epsilon \|\Pi_{S_\gamma}(\theta) - \theta\|\\
  &\leq \DPR(\Pi_{S_\gamma}(\theta),\theta) + L_z\epsilon \gamma\\
  &= \widehat\PR(\theta) + L_z\epsilon\gamma.
\end{align*}
Similarly we obtain $\PR(\theta) \geq \widehat\PR(\theta) - L_z\epsilon\gamma$, which completes the proof.

\subsection{Proof of Proposition \ref{prop:elimination}}

We will show that $\PRLB(\theta) \leq \PR(\theta)$ and $\PRmin\geq \PR(\thetaPO)$; these two facts immediately imply $\Delta(\theta) := \PR(\theta) - \PR(\thetaPO) \geq \PRLB(\theta) - \PRmin$.

The first bound follows because
$\PR(\theta) = \DPR(\theta,\theta) \geq \DPR(\theta',\theta) - L_z\epsilon \|\theta'-\theta\|$ for all $\theta'$, where we use $(L_z\epsilon)$-Lipschitzness of $\DPR$ in the first argument. Similarly, the second bound follows because
\[\PR(\thetaPO) = \min_\theta \DPR(\theta,\theta) \leq \min_\theta (\DPR(\theta',\theta) + L_z\epsilon \|\theta-\theta'\|),\]
for all $\theta'$.

\section{Regret analysis of Algorithm~\ref{alg:adaptive-exploration}}
\label{sec:proofmainthm}

In this section, we prove a regret bound for Algorithm~\ref{alg:adaptive-exploration}. At a high level, Theorem \ref{thm:fine-grained-seq} combines bounds specific to performative prediction with ingredients from the analysis of successive elimination \citep{even2002pac}. First, using a finite-sample analogue of Proposition \ref{prop:elimination}, we show that after phase $p$ all models $\theta \in \mathcal{A}$ have suboptimality $\Delta(\theta) \le 8 \gamma_p$. We then upper bound the number of models in each suboptimality band $\left\{\theta : 16 L_z\epsilon r \le \Delta(\theta) < 32 L_z\epsilon r\right\}$, for fixed $r$, that are deployed in each phase, by leveraging the definition of sequential zooming dimension. The remainder of the proof separately analyzes the regret incurred from the first $\log_2(1/(L_z\epsilon))$ phases, in which the finite-sample error dominates the discretization error, and the regret from the later phases, in which the finite-sample error and the discretization error are of the same order.

We use $\Rph(p_1:p_2)$ to denote the regret incurred from phase $p_1$ to phase $p_2$:
\[\Rph(p_1:p_2) = \E \sum_{p=p_1}^{p_2}\Delta(\theta_p).\]
We let $\Rph(0:p)\equiv \Rph(p)$. For phases $p$ that happen after the time horizon $T$, we assume that the incurred regret is 0; for example, if phases $p_1\leq p_2$ happen after $T$, then $\Rph(p_1:p_2) = 0$.

\subsection{Clean event}

First, we define a clean event that guarantees that the estimates $\widehat{\mathrm{DPR}}(\theta,\theta')$ are close to the true values $\mathrm{DPR}(\theta,\theta')$ at all phases. The clean event essentially guarantees uniform convergence over $\widehat{\mathrm{DPR}}(\theta,\cdot)$ for every $\theta \in \mathcal{P}_p$. 
\begin{definition}[Clean event]
Denote the ``clean event'' by
\begin{equation}
    \label{eq:cleanevent}
    E_{\mathrm{clean}} = \left\{\forall p: \sup_{\theta \in \mathcal{P}_p} \sup_{\theta' \in\Theta}  \left|\widehat{ \mathrm{DPR}}(\theta,\theta') - \DPR(\theta,\theta')\right|\leq \frac{ 2\mathfrak{C}^*(\ell) + 3 \sqrt{\log(T)}}{\sqrt{n_p m_0 }} \right\},
\end{equation}
where $\cP_p$ is the set of all models deployed in phase $p$ during time horizon $T$.
\end{definition}

We show that the clean event occurs with high probability.
\begin{lemma}
\label{lemma:cleanevent}
The clean event holds with high probability,
\[\mathbb{P}\left\{E_{\mathrm{clean}}\right\} \geq 1- T^{-2}.\]
\end{lemma}

\begin{proof}
We consider each interval of length $n_p$ in phase $p$, during which the same model is deployed, separately, and then take a union bound over these intervals across all phases. Therefore, we will say interval $s$ in phase $p$ to refer to steps $(s-1)n_p+1,\dots,s n_p$ in phase $p$. For the sake of this proof, we consider a ``counterfactual'' set of samples for each model $\theta$ that augments the set of actually observed samples. In particular, for interval $s$ in phase $p$, we let $\{z^{\theta,s}_j\}_{j=1}^{n_p m_0}$ denote i.i.d. samples from $\cD(\theta)$. The samples for different time intervals and different phases are independent. When model $\theta$ is deployed, we observe the samples corresponding to the interval in which $\theta$ is deployed.

For each phase $p$ and each time interval $s$ within phase $p$, let $E^{s,p}_{\text{end}}$ denote the event that phase $p$ terminates strictly before interval $s$ is reached. Let $E^{s,p}_{\mathrm{clean}}$ denote the event that one of the following two holds:
\begin{enumerate}[itemsep=1pt]
    \item[(E1)] $E^{s,p}_{\text{end}}$ occurs;
    \item[(E2)] $E^{s,p}_{\text{end}}$ does not occur, and for the model $\theta_s$ deployed in time interval $s$ it holds that:
\[\sup_{\theta' \in\Theta}  \left|\widehat{\mathrm{DPR}}(\theta_s,\theta') - \DPR(\theta_s,\theta')\right|\leq \frac{ 2\mathfrak{C} + 3 \sqrt{\log(T)}}{\sqrt{n_p m_0 }}, \]
\end{enumerate}
where $\theta_s$ is a random variable.

The probability that  $E^{s,p}_{\mathrm{clean}}$ does not occur is at most:
\begin{align*}
& \mathbb{P}\left[\lnot E^{s,p}_{\text{end}} \And \sup_{\theta' \in\Theta}  \left|\widehat{ \mathrm{DPR}}(\theta_s,\theta') - \DPR(\theta_s,\theta')\right| > \frac{ 2\mathfrak{C} + 3 \sqrt{\log(T)}}{\sqrt{n_p m_0 }}\right] \\
&= \mathbb{P}\left[\lnot E^{s,p}_{\text{end}} \right] \cdot \mathbb{P}\left[ \sup_{\theta' \in\Theta}  \left|\widehat{ \mathrm{DPR}}(\theta_s,\theta') - \DPR(\theta_s,\theta')\right| > \frac{ 2\mathfrak{C} + 3 \sqrt{\log(T)}}{\sqrt{n_p m_0 }} \Bigg| \lnot E^{s,p}_{\text{end}}\right] \\
&\le \mathbb{P}\left[ \sup_{\theta' \in\Theta}  \left|\widehat{ \mathrm{DPR}}(\theta_s,\theta') - \DPR(\theta_s,\theta')\right| > \frac{ 2\mathfrak{C} + 3 \sqrt{\log(T)}}{\sqrt{n_p m_0 }}  \Bigg| \lnot E^{s,p}_{\text{end}}\right].
\end{align*}
We can equivalently write this as
\begin{align*}
&\mathbb{P}\left[ \sup_{\theta' \in\Theta}  \left|\frac{1}{n_p m_0} \sum_{j=1}^{n_p m_0} \ell(z^{\theta_s,s}_j;\theta') - \DPR(\theta_s,\theta')\right| > \frac{ 2\mathfrak{C} + 3 \sqrt{\log(T)}}{\sqrt{n_p m_0 }}  \Bigg| \lnot E^{s,p}_{\text{end}}\right] \\ \\
&= \mathbb{E}_{\theta \sim \theta_s}\left[\mathbb{P}\left[ \sup_{\theta' \in\Theta}  \left|\frac{1}{n_p m_0} \sum_{j=1}^{n_p m_0} \ell(z^{\theta,s}_j;\theta') - \DPR(\theta,\theta')\right| > \frac{ 2\mathfrak{C} + 3 \sqrt{\log(T)}}{\sqrt{n_p m_0 }}  \Bigg| \lnot E^{s,p}_{\text{end}}, \theta_s = \theta\right]\right]. \\
\end{align*}
To upper bound this expression, it suffices to show an upper bound on 
\[\mathbb{P}\left[ \sup_{\theta' \in\Theta}  \left|\frac{1}{n_p m_0} \sum_{j=1}^{n_p m_0} \ell(z^{\theta,s}_j;\theta') - \DPR(\theta,\theta')\right| > \frac{ 2\mathfrak{C} + 3 \sqrt{\log(T)}}{\sqrt{n_p m_0 }} \Bigg| \lnot E^{s,p}_{\text{end}}, \theta_s = \theta\right] \]
that holds for every $\theta$. The first observation is that for any $\theta$, the samples $\{z^{\theta,s}_j\}_{j=1}^{n_p m_0}$ are independent of the event $\left\{\theta_s = \theta, \lnot E^{s,p}_{\text{end}} \right\}$, since the event depends only on the samples collected in previous time intervals and phases. This means that the above probability is equal to:
\[\mathbb{P}\left[\sup_{\theta' \in\Theta}  \left|\frac{1}{n_p m_0} \sum_{j=1}^{n_p m_0} \ell(z^{\theta,s}_j;\theta') - \DPR(\theta,\theta')\right| > \frac{ 2\mathfrak{C} + 3 \sqrt{\log(T)}}{\sqrt{n_p m_0 }}\right]. \] 
Let $\epsilon_j$ denote i.i.d. Rademacher random variables. Then, we can observe that with probability $1 - T^{-3}$, it holds that:
\begin{align*}
\sup_{\theta'\in\Theta} \left|\frac{1}{n_p m_0} \sum_{j=1}^{n_p m_0} \ell(z^{\theta,s}_j;\theta') - \DPR(\theta,\theta')\right|
&\leq \E \left[\sup_{\theta'\in\Theta} \left|\frac{1}{n_p m_0} \sum_{j=1}^{n_p m_0} \ell(z^{\theta,s}_j;\theta') - \DPR(\theta,\theta')\right|\right] + \sqrt{\frac{6 \log(T)}{n_p m_0}} \\
&\leq 2 \cdot \E \left[\sup_{\theta'\in\Theta} \left|\frac{1}{n_p m_0} \sum_{j=1}^{n_p m_0} \ell(z^{\theta,s}_j;\theta') \cdot \epsilon_j\right|\right] + \sqrt{\frac{6 \log(T)}{n_p m_0}} \\
&\leq \frac{2}{\sqrt{n_p m_0}} \cdot \sup_{n \ge 1} \sqrt{n} \E \left[\sup_{\theta'\in\Theta} \left|\frac{1}{n} \sum_{j=1}^{n} \ell(z^{\theta}_j;\theta') \cdot \epsilon_j\right|\right] + \sqrt{\frac{6 \log(T)}{n_p m_0}} \\
&\leq \frac{2 \mathfrak{C}^*(\ell) + 3 \sqrt{\log(T)}}{\sqrt{n_p m_0 }},
\end{align*}
where the first step follows from the bounded differences inequality and the second step follows from a classical symmetrization argument. In the penultimate step we let $\{z_j^{\theta}\}_{j\in\mathbb{N}}$ denote an infinite sequence of samples from $\cD(\theta)$. Putting this all together, we have that:
\[1 - \mathbb{P}\left[E^{s,p}_{\mathrm{clean}}\right] \le T^{-3}. \] 

Finally, using that there are at most $T$ intervals before time horizon $T$ (across all phases), by a union bound we see that:
\[1 - \mathbb{P}\left[E_{\mathrm{clean}}\right] \le T^{-2},  \] as desired.

\end{proof}

\subsection{Suboptimality of the active set}

We show that the elimination strategy in Algorithm \ref{alg:adaptive-exploration} will never eliminate any performatively optimal point.
\begin{lemma}
\label{lemma:Anonempty}
On the clean event \eqref{eq:cleanevent}, any performatively optimal point $\thetaPO \in \argmin_{\theta} \PR(\theta)$ will always remain in $\mathcal{A}$. 
\end{lemma}

\begin{proof}
It suffices to show that $\thetaPO$ cannot be eliminated in Step 14 of Algorithm \ref{alg:adaptive-exploration}. Fix any phase $p$ and denote by 
$\cP_p$ the running set of deployed points at any point during phase $p$. Then, we have:
\begin{align*}
\PRLB(\thetaPO) &= \max_{\theta' \in \mathcal{P}_p} \left( \widehat{\DPR}(\theta', \thetaPO) - L_z \epsilon \|\thetaPO - \theta'\|\right)\\
&\leq \max_{\theta' \in \cP_p} \left( \DPR( \theta', \thetaPO) - L_z \epsilon \|\thetaPO - \theta'\|\right) + \gamma_p \\
&\leq \PR(\thetaPO)  + \gamma_p\\
&= \min_\theta \DPR(\theta,\theta)  + \gamma_p\\
&\leq \min_\theta \min_{\theta'\in\cP_p} \DPR(\theta',\theta) + L_z\epsilon \|\theta-\theta'\|  + \gamma_p\\
&\leq \min_{\theta}\min_{\theta' \in \mathcal{P}_p}  \widehat{\DPR}( \theta',\theta) + L_z \epsilon \|\theta' - \theta\|  + 2\gamma_p \\
&= \PRmin + 2\gamma_p.
\end{align*}
Therefore, $\PRLB(\thetaPO) \leq \PRmin + 2\gamma_p$, implying that $\thetaPO$ cannot be removed from $\cA$ during phase $p$. Since this is true for any phase $p$, that completes the proof of the lemma.
\end{proof}

We next show that the elimination strategy is sufficiently effective that all models that remain active after a given phase $p$ have suboptimality at most $8 \gamma_p$. 
\begin{lemma} 
\label{lemma:approx_sequential}
On the clean event \eqref{eq:cleanevent},  after phase $p$ all models $\theta\in\cA$ satisfy $\Delta(\theta) \le 8\gamma_p$. 
\end{lemma}
\begin{proof}
Fix a phase $p$.
We will analyze ${\cP_p}$ at the \emph{end} of phase $p$. The proof relies on two key facts:
\begin{enumerate}[itemsep=1pt]
\item[(F1)] If $\theta$ is active after phase $p$, then  $\|\theta - \Pi_{\cP_p}(\theta)\| \leq r_p$, where $\Pi_{\cP_p}(\theta) = \argmin_{\theta'\in\cP_p}\|\theta - \theta'\|$.\label{fact1}
\item[(F2)] $\thetaPO$ is active after phase $p$.
\label{fact2}
\end{enumerate}

The first fact follows since during phase $p$ net points cannot be eliminated from $\cS_p$ in Step~13 while some parameter within an $r_p$-neighborhood is active. The second fact is proved in Lemma \ref{lemma:Anonempty}. Note that from fact (F1) it further follows that there is always a model in $\cP_p$ within the $r_p$-neighborhood of $\thetaPO$.

Now suppose that $\theta$ is active after phase $p$. Then,  we have:
\begin{align*}
    \PR(\theta) &\leq \DPR(\Pi_{\cP_p}(\theta), \theta) + L_z \epsilon \|\Pi_{\cP_p}(\theta) - \theta\| \\
    &\le \widehat{\DPR}(\Pi_{\cP_p}(\theta), \theta) + L_z \epsilon \|\Pi_{\cP_p}(\theta) - \theta\| + \gamma_p\\
    &\leq \min_{\theta'} \left( \widehat{\DPR}(\Pi_{\cP_p}(\theta'), \theta') + L_z \epsilon \| \Pi_{\cP_p}(\theta') - \theta'\|\right) + 2L_z \epsilon \|\Pi_{\cP_p}(\theta) - \theta\| + 3\gamma_p,
\end{align*}
where we used the definitions of $\PRmin$ and $\PRLB(\theta)$, together with the fact that $\PRLB(\theta)\leq \PRmin + 2 \gamma_p$ for active models. Now choosing $\theta'=\thetaPO$, applying (F1), (F2), and accounting for finite-sample uncertainty we find
\begin{align*}
   \PR(\theta) &\leq \widehat{\DPR}(\Pi_{\cP_p}(\thetaPO), \thetaPO) + L_z\epsilon \| \Pi_{\cP_p}(\thetaPO) - \thetaPO\| + 2L_z \epsilon \|\Pi_{\cP_p}(\theta) - \theta\| + 3\gamma_p  \\
    &\leq \widehat{\DPR}(\Pi_{\cP_p}(\thetaPO), \thetaPO)   + 3L_z\epsilon r_p + 3\gamma_p\\
    &\le \DPR(\thetaPO, \thetaPO)  + L_z \epsilon \|\Pi_{\cP_p}(\thetaPO) - \thetaPO\|  + 3L_z\epsilon r_p + 4\gamma_p\\
    &\leq \PR(\thetaPO) + 4(L_z\epsilon r_p+ \gamma_p)\\
    &= \PR(\thetaPO) + 8 \gamma_p,
\end{align*}
where we use the fact that $r_p = \frac {\gamma_p}{L_z\epsilon}$. Rearranging the terms we obtain $\Delta(\theta) = \PR(\theta) - \PR(\thetaPO) \leq 8\gamma_p$ as claimed in Lemma~\ref{lemma:approx_sequential}.
\end{proof}

\subsection{Bounding the number of suboptimal deployments}

For $i \ge 1$, we consider the suboptimality bands 
\[\cE_i = \left\{  \theta : \Delta(\theta) \in [8 \cdot 2^{-i} L_z \epsilon, 16 \cdot 2^{-i} L_z \epsilon)\right\}.\]

In the following lemma, we bound the number of times that models in $\cE_i$ can be deployed in a given phase.

\begin{lemma}
\label{lemma:armscountsequential}
Suppose that the clean event \eqref{eq:cleanevent} holds. For $i\geq 1$, in phase $\log_2(1/(L_z\epsilon)) \leq p \leq  \log_2(1/(L_z\epsilon)) + i + 1$, the number of models in $\cE_i$ that are deployed is at most $\cO\left((3/r_p)^{d} \right)$ in expectation, where $d$ is the $(L_z \epsilon)$-sequential zooming dimension.
\end{lemma}

To provide intuition for Lemma \ref{lemma:armscountsequential}, it is informative to consider a weaker version of the lemma where $d$ is taken to be the $(L_z \epsilon)$-zooming dimension rather than the $(L_z \epsilon)$-sequential zooming dimension. To see why this weaker version of the lemma is true, notice that at the beginning of phase $p$, the set of active models $\mathcal{A}$ is a subset of $\left\{\theta : \Delta(\theta) \le 8 \gamma_{p-1}\right\} = \left\{\theta : \Delta(\theta) \le 16 \gamma_{p}\right\}  = \left\{\theta : \Delta(\theta) \le 16 L_z \epsilon r_p \right\}$. The set of models deployed in phase $p$ is contained in a minimal $r_p$-net of $\mathcal{A}$. Notice that $r_p \ge 2^{-(i+1)}$. By the definition of zooming dimension, we know that at most a multiple of $\left(\frac{3}{r_p} \right)^d$ elements from the set $\left\{  \theta : \Delta(\theta) \in [8 \cdot 2^{-i} L_z \epsilon, 16 \cdot 2^{-i} L_z \epsilon)\right\} =\left\{  \theta : \Delta(\theta) \in [16 \cdot 2^{-(i+1)} L_z \epsilon, 32 \cdot 2^{-(i+1)} L_z \epsilon)\right\} $ are deployed, as desired.   

The proof of Lemma \ref{lemma:armscountsequential} boils down to refining this proof sketch to account for the sequential elimination aspect of Algorithm \ref{alg:adaptive-exploration}.
\begin{proof}
For the purposes of this analysis, we condition on the clean event. 

Fix a phase $\log_2(1/(L_z\epsilon)) \leq p \leq  \log_2(1/(L_z\epsilon)) + i + 1$. Let $\mathcal{S}_p^0$ be the covering of $\cA$ chosen at the beginning of phase $p$, and let $\pi$ be an ordering of $\cS_p^0$ chosen uniformly at random. It is not difficult to see that Algorithm \ref{alg:adaptive-exploration} is equivalent to drawing $\pi$ at the beginning of the phase, and deploying models in the order given by $\pi$ (naturally, skipping those that get eliminated). For technical convenience, we analyze this reformulation of the algorithm. 

Condition on a realization $\pi$, and let $\cP_p \subseteq \mathcal{S}_p^0$ be the set of models that are ultimately get deployed. Note that $\cP_p$ depends on the randomness arising from finite-sample noise at each step of the phase. We will show a bound on $|\cP_p|$ that deterministically holds on the clean event. In particular, consider the models $\theta \in \cS_p^0 \cap \{\theta: 8 L_z \epsilon r_i \leq \Delta(\theta) < 16 L_z \epsilon r_i \}$ such that:
\begin{equation}
\label{eq:seqzoomcond}
  \PR^{r_p}_{\mathrm{LB}}(\pi(\theta)) \leq \PRmin(\pi(\theta)) + 4 L_z \epsilon r_p = \PRmin(\pi(\theta)) + 4 \gamma_p.
\end{equation}
We will show that $\cP_p$ is a subset of such models. 

Suppose that $\theta_{\text{net}}\in \cS_p^0$ is deployed in phase $p$. Then, that means that there exists $\theta'' \in \text{Ball}_{r_p}(\theta_{\text{net}})$ that remains active after the first $\pi(\theta_{\text{net}})-1$ deployments; that is:
\begin{align*}
   \max_{\theta' : \pi(\theta') < \pi(\theta_{\text{net}})} (\widehat{\DPR}(\theta', \theta'') - L_z \epsilon \|\theta' - \theta''\|) &= \PRLB(\theta'') \\
   &\le \PRmin + 2\gamma_p \\
   &= \min_{\theta} \min_{\theta' : \pi(\theta') < \pi(\theta_{\text{net}})} (\widehat{\DPR}(\theta', \theta) + L_z \epsilon \|\theta' - \theta\|) + 2\gamma_p. 
\end{align*}
Since the clean event holds, we know that:
\[\max_{\theta' : \pi(\theta') < \pi(\theta_{\text{net}})} \left(\text{DPR}(\theta', \theta'') - L_z \epsilon \|\theta' - \theta''\|\right) - \gamma_p  \le \min_{\theta} \min_{\theta' : \pi(\theta') < \pi(\theta_{\text{net}})} (\text{DPR}(\theta', \theta) + L_z \epsilon \|\theta' - \theta\|) + 3 \gamma_p. \]
Rearranging, this means that:
\begin{align*}
    \max_{\theta' : \pi(\theta') < \pi(\theta_{\text{net}})} &(\text{DPR}(\theta', \theta'') - L_z \epsilon \|\theta' - \theta''\|) \\&  \le \min_{\theta} \min_{\theta' : \pi(\theta') < \pi(\theta_{\text{net}})} (\text{DPR}(\theta', \theta) + L_z \epsilon \|\theta' - \theta\|) + 4 \gamma_p = \PRmin(\pi(\theta_{\text{net}})) + 4 \gamma_p. 
\end{align*}
This further implies that:
\[\PR^{r_p}_{\mathrm{LB}}(\pi(\theta_{\text{net}})) = \min_{\theta'' \in \text{Ball}_{r_p}(\theta_{\text{net}})} \max_{\theta' : \pi(\theta') < \pi(\theta_{\text{net}})} \left(\text{DPR}(\theta', \theta'') - L_z \epsilon \|\theta' - \theta''\|\right) \le \PRmin(\pi(\theta_{\text{net}})) + 4 \gamma_p. \]

We see that any $\theta_{\text{net}} \in \cP_p$ must satisfy condition \eqref{eq:seqzoomcond}.
By the definition of sequential zooming dimension, we know that the expected number of models in $\cE_i$ that satisfy \eqref{eq:seqzoomcond}, where the expectation is taken over the randomness of $\pi$, is at most a multiple of $\left(\frac{3}{r_p}\right)^d$, hence $\E |\cP_p \cap \cE_i| \leq \cO\left( \left(\frac{3}{r_p}\right)^d\right)$, as desired.
\end{proof}

\subsection{Regret bound on the clean event}

To bound the regret on the clean event, we break the analysis into two cases: (a) the first $\log_2(1/(L_z\epsilon))$ phases, and (b) all remaining phases.

\begin{lemma}
\label{lemma:phase0}
Suppose that the clean event \eqref{eq:cleanevent} holds. In the first $\lfloor\log_2(1/(L_z\epsilon))\rfloor$ phases, the algorithm has incurred regret at most
\[\Rph\left(\lfloor\log_2(1/(L_z\epsilon))\rfloor\right) = \cO \left(\sqrt{\frac{T}{m_0}} \left(\sqrt{\log T} + \mathfrak{C}\right) \right) .\]
\end{lemma}
\begin{proof}
During phases $p \leq \log_2(1/(L_z\epsilon))$, we deploy a single model since $r_p \geq 1$ and $\Theta$ is assumed to have radius $1$.

We break the first $\left\lfloor \log_2(1/(L_z\epsilon))\right\rfloor$ phases into two cases. For a value of $N\geq 0$ specified later, we consider cases $p<N$ and $p\geq N$ separately.

\paragraph{Case 1: phases $N\leq p\leq \lfloor\log_2(1/(L_z\epsilon))\rfloor$.}  By Lemma \ref{lemma:approx_sequential}, we see that the model deployed in phase $N$ must have suboptimality at most $8 \cdot 2^{-N + 1} = 2^{-N + 4}$. Since the algorithm runs for at most $T$ time steps, this means that the total regret incurred in these phases is at most $ T \cdot 2^{-N + 4}$.

\paragraph{Case 2: phases $0\leq p< \min\{N,\lfloor\log_2(1/(L_z\epsilon))\rfloor\}$.}

By Lemma \ref{lemma:approx_sequential}, we know that the model deployed in phase $p$ must have suboptimality at most $8 \cdot 2^{-p + 1} = 2^{-p + 4}$. Moreover, this model is deployed for $n_p = \left\lceil \frac{\left(2 \mathfrak{C} + 3 \sqrt{\log T}\right)^2}{\gamma_p^2 m_0}\right\rceil$ steps. 
The regret incurred up to phase $N$ can thus be bounded as:
\begin{align*}
  \Rph(N) &\leq \sum_{p=0}^{N-1} n_p  2^{-p + 4}\\
  &\le  16 \sum_{p=0}^{N-1} 2^{-p}\left\lceil \frac{ 2^{2p}(2 \mathfrak{C} + 3 \sqrt{\log T})^2}{m_0}\right\rceil.
  \end{align*}
  Since we assume $m_0 = o((\mathfrak{C} + \sqrt{\log T})^2)$, for a large enough $T$ we have $n_p\geq 1$ and thus $\lceil n_p\rceil \leq 2 n_p$. Therefore,
  \begin{align*}
  \Rph(N) &\le C \sum_{p=0}^{N-1} 2^{-p} \frac{2^{2p}(2 \mathfrak{C} + 3 \sqrt{\log T})^2}{m_0} \\
  &\le C  \frac{(2 \mathfrak{C} + 3 \sqrt{\log T})^2}{m_0} \left(\sum_{p=0}^{N-1} 2^p \right) \\
  &\le C \cdot 2^{N} \frac{(2 \mathfrak{C} + 3 \sqrt{\log T})^2}{m_0},
\end{align*}
for some large enough constant $C>0$.

Putting the two cases together, on the clean event, the total regret incurred in phases $p=0,\dots,\lfloor\log_2(1/(L_z\epsilon))\rfloor$ can be upper bounded by
\[ C \cdot 2^{N} \frac{(2 \mathfrak{C} + 3 \sqrt{\log T})^2}{m_0} +  T \cdot 2^{-N + 4}.\] 

We can also trivially upper bound the regret by $T$, using the fact that the loss incurred at each step is at most $1$. This means that we obtain a regret bound of:
\[ \cO\left(\min\left\{T, 2^{N} \frac{(2 \mathfrak{C} + 3 \sqrt{\log T})^2}{m_0} +  T \cdot 2^{-N + 4}\right\}\right).\] 

We now choose $N$ to minimize this bound. We let $\eta = 2^{-N}$ and optimize over $\eta\in(0,1)$. Optimizing over $\eta$ instead of an integral value of $N$ changes the bound by constant factors at most.
This means that we can upper bound the regret by:
\[\cO\left(\min_{0 < \eta \le 1}  \min\left\{T, \eta^{-1} \frac{\left(2 \mathfrak{C} + 3 \sqrt{\log T}\right)^2}{ m_0} + T \eta\right\}\right).\]
If $\eta > 1$, then the minimum of the two terms would be $T$, which is at least as big as the above expression. Therefore, we can upper bound the above expression by:
\[\cO\left(\min_{\eta >0} \left( \eta^{-1} \frac{\left(2 \mathfrak{C} + 3 \sqrt{\log T}\right)^2}{ m_0} + T \eta\right)\right).\]
We set $\eta = \frac{3\sqrt{\log T} + 2\mathfrak{C}}{\sqrt{m_0 T}}$
and obtain a regret bound of:
\[\Rph\left(\lfloor\log_2(1/(L_z\epsilon))\rfloor\right) =\cO \left(\sqrt{\frac{T}{m_0}} \left(\sqrt{\log T} + \mathfrak{C} \right)\right),\]
as desired.
\end{proof}

\begin{lemma}
Suppose that the clean event \eqref{eq:cleanevent} holds. Let $d \ge 0$ be such that for every $i \ge 0$ and every phase $p \in [\log_2(1/(L_z\epsilon)), \log_2(1/(L_z\epsilon)) + i + 1]$, the number of models in $\cE_i = \left\{  \theta : \Delta(\theta) \in [2^{-i + 3} L_z \epsilon, 2^{-i+4} L_z \epsilon]\right\}$ that are deployed in phase $p$ is upper bounded by $\cO\left(\left(\frac{3}{r_p} \right)^d\right)$ in expectation. Then, the regret incurred in phases $p \ge \log_2(1/(L_z \epsilon))$, within time horizon $T$, can be upper bounded as 
\begin{equation*}
\Rph\left(\lceil\log_2(1/(L_z \epsilon))\rceil:\infty\right) \leq \cO\left(T^{\frac{d+1}{d+2}} (L_z \epsilon)^{\frac{d}{d+2}} \left(\frac{(\sqrt{\log T} + \mathfrak{C})^2}{m_0}\right)^{\frac{1}{d+2}}\right).
\end{equation*}
\label{lemma:regrettarget}
\end{lemma}

\begin{proof}
By Lemma \ref{lemma:approx_sequential}, we see that all models $\theta$ that are active in phase $p=\lceil\log_2(1/L_z\epsilon))\rceil$ or later have $\Delta(\theta) \le 8 L_z \epsilon r_p \leq 8L_z\epsilon$. We split these models into suboptimality bands and define, for each $i \ge 1$, the set:
\[\cE_i = \left\{  \theta : \Delta(\theta) \in [8 \cdot 2^{-i} L_z \epsilon, 16 \cdot 2^{-i} L_z \epsilon)\right\}.\]
Note that all models deployed starting with phase $\lceil\log_2(1/(L_z \epsilon))\rceil$ are in $\cup_{i \ge 1} \cE_i$. For a value of $N$ specified later, we break the analysis into two cases.

\paragraph{Case 1:  models in $\cup_{i > N} \cE_i$.} Since the algorithm runs for at most $T$ time steps, the total regret incurred due to deploying models in $\cup_{i > N} \cE_i$ is at most 
\begin{equation*}
    T \cdot 16 \cdot 2^{-N-1}   L_z \epsilon \le 8 T 2^{-N} L_z \epsilon.
\end{equation*}
\paragraph{Case 2: models in $\cup_{1 \le i \le N} \cE_i$.} 
By Lemma \ref{lemma:approx_sequential}, we know that all models $\theta$ that are active is phases $p \geq N + \log_2(1/L_z\epsilon)$ have $\Delta(\theta) \le 8 2^{-p} = 8 \cdot 2^{-N} L_z \epsilon = 16 \cdot 2^{-N - 1} L_z \epsilon$. This means that all models that are active after phase $N + \log_2(1/L_z\epsilon)$ are in $\cup_{i > N} \cE_i$. Thus, to bound the regret incurred by deploying models in $\cup_{1 \le i \le N} \cE_i$ in phase $\lceil\log_2(1/L_z\epsilon)\rceil$ or later, we only need to consider phases $p = \lceil\log_2(1/L_z\epsilon)\rceil,\dots,N + \log_2(1/L_z\epsilon)$. 

For $1 \le i \leq N$, consider $\cE_i$. 
By Lemma \ref{lemma:approx_sequential}, we know that any $\theta \in \cE_i$ can only be active during phases $p \le \log_2(1/L_z\epsilon) + i + 1$. By assumption, in phase $p$, the number of points in $\cE_i$ that are deployed is at most of the order $\left(\frac{3}{r_p} \right)^d$ in expectation. Moreover, each point is deployed $n_p$ times. Putting this all together, the expected number of points in $\cE_i$ deployed in phase $p$ is at most: 
\begin{equation*}
    \cO\left(\left(\frac{3}{r_p}\right)^d n_p\right) = \cO\left(\left(\frac{3}{r_p}\right)^d \frac{(2 \mathfrak{C} + 3 \sqrt{\log T})^2}{L_z^2 \epsilon^2 r_p^2 m_0}\right),
\end{equation*}
where we use the fact that, given the condition $m_0 = o((\mathfrak{C} + \sqrt{\log T})^2)$, $n_p \geq 1$ for large enough $T$ and hence we can bound $\lceil n_p \rceil \leq 2n_p$. Take $p = j + \log_2(1/L_z\epsilon))$; then, $r_p = 2^{-j}$. We sum over phases $\log_2(1/(L_z\epsilon)) \le p \le \log_2(1/(L_z\epsilon)) + i +1$ to obtain that in expectation, the total number of times that these models are deployed is at most: 
\begin{align*}
 \cO\left(\frac{3^d(2 \mathfrak{C} + 3 \sqrt{\log T})^2}{L_z^2 \epsilon^2 m_0}\sum_{j=0}^{i+1} 2^{j(d+2)} \right) = \cO\left(\frac{3^d(2 \mathfrak{C} + 3 \sqrt{\log T})^2}{L_z^2 \epsilon^2 m_0} 2^{(i+1)(d+2)} \right).
\end{align*}
Using the fact that the models have suboptimality at most $16 \cdot 2^{-i} L_z \epsilon = 32 \cdot 2^{-(i+1)} L_z \epsilon$, we see that the regret incurred by deploying models in $\cE_i$ is upper bounded by:
\[ \cO\left(\frac{3^d(2 \mathfrak{C} + 3 \sqrt{\log T})^2}{L_z \epsilon m_0} 2^{(i+1)(d+1)} \right). \] 
We sum over $1 \le i \leq N$ to obtain the total regret incurred due to deploying models in $\cup_{1\leq i\leq N}\cE_i$:
\[\cO\left(\frac{3^d(2 \mathfrak{C} + 3 \sqrt{\log T})^2}{L_z \epsilon m_0} 2^{(N+2)(d+1)} \right). \]

Putting together the two cases we obtain a total regret bound of 
\[\cO\left(\frac{3^d(2 \mathfrak{C} + 3 \sqrt{\log T})^2}{L_z \epsilon m_0} 2^{(N+2)(d+1)} + T 2^{-N} L_z \epsilon\right).\] 

We also can upper bound the regret by $8 T L_z \epsilon$, since all models active after phase $\lfloor\log_2(1/(L_z\epsilon))\rfloor$ have $\Delta(\theta) \le 8 L_z\epsilon$ and there are at most $T$ time steps in total. This means that we can bound the regret by:
\[\cO\left(\min \left\{T L_z\epsilon, \frac{3^d(2 \mathfrak{C} + 3 \sqrt{\log T})^2}{L_z \epsilon m_0} 2^{(N+2)(d+1)} + T 2^{-N} L_z \epsilon\right\}\right). \] 

We now choose $N$ to minimize this bound. We let $\eta = 2^{-N}$ and choose some $\eta\in(0,1)$. The error from optimizing over $\eta\in (0,1)$ instead of an integral value of $N$ contributes at most constant factors.
This means that we can upper bound the regret by:
\[\cO\left(\min \left\{T L_z\epsilon, \frac{12^d(2 \mathfrak{C} + 3 \sqrt{\log T})^2}{L_z \epsilon m_0} \eta^{-(d+1)}  + T \eta L_z \epsilon\right\}\right), \] 
for any $\eta\in(0,1)$. Note that, if $\eta \geq 1$, the second term in the bound is at least as large as the first term, hence we can choose any $\eta>0$. In particular, we can further upper bound the regret by
\[\cO\left(\min_{\eta>0}\left( \frac{12^d(2 \mathfrak{C} + 3 \sqrt{\log T})^2}{L_z \epsilon m_0} \eta^{-(d+1)}  + T \eta L_z \epsilon\right) \right). \] 
Now, we set
\[\eta = \left(\frac{12^d\left(3\sqrt{\log T} + 2\mathfrak{C}\right)^2}{T L_z^2 \epsilon^2m_0}\right)^{\frac{1}{d+2}}.\]
Thus, we finally get a regret bound of
\[\cO\left(T^{\frac{d+1}{d+2}} (L_z \epsilon)^{\frac{d}{d+2}} \left(\frac{\left(\sqrt{\log T} + \mathfrak{C}\right)^2}{m_0}\right)^{\frac{1}{d+2}}\right), \]
as desired.
\end{proof}

\subsection{Proof of Theorem \ref{thm:fine-grained-seq}}

Now, we are ready to prove Theorem \ref{thm:fine-grained-seq}.

First, we handle the case where the clean event defined in \eqref{eq:cleanevent} does not hold and the concentration bound is violated. By Lemma \ref{lemma:cleanevent}, this happens with probability at most $T^{-2}$. The regret incurred in each deployment is at most $1$ and there are $T$ deployments, so these events contribute a negligible factor $T^{-1}$ to the expected regret. 

For the case where the clean event holds we can build on Lemma \ref{lemma:armscountsequential}, Lemma~\ref{lemma:phase0}, and Lemma~\ref{lemma:regrettarget}. From Lemma \ref{lemma:phase0}, we obtain a bound for the total regret incurred in phases up to $\lfloor\log_2(1/(L_z \epsilon))\rfloor$. By Lemma \ref{lemma:armscountsequential} we can set the parameter $d$ in Lemma~\ref{lemma:regrettarget} to be the $(L_z\epsilon)$-sequential zooming dimension, and thus from Lemma~\ref{lemma:regrettarget} we obtain a regret bound for all later phases.

Putting all this together yields the desired bound.



\section{Regret analysis of Algorithm \ref{alg:location-families}}

The proof of Theorem \ref{thm:lin_bound} relies on two key lemmas. One proves that $\mathcal{C}_t$ are valid confidence sets for $\mu_*$ at every step, and the other one proves a regret bound assuming that $\mathcal{C}_t$ are valid confidence sets.

Throughout we denote by $\mathcal{B}_m$ the unit ball in $\mathbb{R}^m$. For a vector $x$ and matrix $M$, we will use the notation $\|x\|_M = \sqrt{x^\top M x}$.

An important object in the proofs will be $S_t := \sum_{i=1}^t \theta_i \bar z_{0,i}^\top$, where $\bar z_{0,i} = \frac{1}{m_0} \sum_{j=1}^{m_0} z_i^{(j)} - \mu_*^\top \theta_i$. Essentially $\bar z_{0,i}$ is the average over $m_0$ samples from $\cD_0$, collected at step $i$. We will also denote $V_t(\lambda) = (\lambda I + \sum_{i=1}^t \theta_i \theta_i^\top)$, for an arbitrary offset $\lambda >0$, and $V_t \equiv V_t(0)$. Note that in the algorithm statement we use $\Sigma_t = V_t\left(\frac{1}{m_0}\right)$.

\subsection{Clean event}

As for Algorithm \ref{alg:adaptive-exploration}, we introduce a clean event. In this case, the clean event will be defined as
\begin{equation}
\label{eq:clean_linucb}
E_{\mathrm{clean}} = \{\forall t\in \mathbb{N}: \mu_* \in \mathcal{C}_t\},
\end{equation}
where $\mathcal{C}_t$ are the confidence sets constructed in Algorithm \ref{alg:location-families}.

The technical subtlety lies in the fact that the points $\theta_t$ are chosen adaptively, hence one cannot simply apply standard least-squares confidence intervals to argue that the sets $\mathcal{C}_t$ are valid. The same difficulty is resolved in the analysis of the LinUCB algorithm and our proof builds on the proof technique of that analysis.

Before stating the main technical lemma, we start with an auxiliary result that we will use in the proof.

\begin{lemma}
\label{lemma:supmg}
Suppose that $\cD_0$ is $1$-subgaussian. Then, for all $x\in\mathcal{B}_m$ and $y\in\mathbb{R}^{d_\Theta}$, the process \[M_t(x,y) = \exp\left(y^\top S_t x - \frac{1}{2m_0}\|y\|_{V_t }^2\right)\] is a supermartingale with respect to the natural filtration, with $M_0(x,y) = 1$.
\end{lemma}

\begin{proof}
Since $\bar z_{0,i}$ are $\frac{1}{\sqrt{m_0}}$-subgaussian, we know that all one-dimensional projections are also $\frac{1}{\sqrt{m_0}}$-subgaussian, hence $\bar z_{0,i}^\top x$ are independent $\frac{1}{\sqrt{m_0}}$-subgaussian as well. Using this, we know
\[\mathbb{E}\left[\exp(y^\top \theta_t z_{0,t}^\top x)~\Big|~\mathcal{F}_{t-1}\right] \leq \exp\left(\frac{(y^\top \theta_t)^2}{2m_0}\right) = \exp\left(\frac{\|y\|^2_{\theta_t\theta_t^\top}}{2m_0}\right)\]
almost surely. Hence,
\begin{align*}
    \mathbb{E}[M_t(x,y)~|~\mathcal{F}_{t-1} ] &= \mathbb{E}\left[\exp\left(y^\top S_t x - \frac{1}{2m_0}\|y\|_{V_t }^2\right) ~\Big|~ \mathcal{F}_{t-1}\right]\\
    &= M_{t-1}(x,y) \mathbb{E}\left[\exp\left(y^\top \theta_t z_{0,t}^\top x - \frac{1}{2m_0}\|y\|_{\theta_t\theta_t^\top}^2\right)~\Big|~\mathcal{F}_{t-1}\right]\\
    &\leq M_{t-1}(x,y)
\end{align*}
almost surely. Furthermore, $M_0(x,y)= 1$ is trivially true.
\end{proof}

Now we are ready to state the main technical lemma about the validity of $\mathcal{C}_t$.

\begin{lemma}
\label{lemma:conf_regions}
We have that
\[\mathbb{P}\left\{E_{\mathrm{clean}}\right\} \geq 1 - T^{-2}.\]
\end{lemma}

\begin{proof}
First we will show that for any $\delta\in(0,1)$,
\begin{equation}
\label{eq:St-bound}
\mathbb{P}\left\{\exists t \in \mathbb{N}: \|V_t(\lambda)^{-1/2}S_t\|^2 \geq \frac{1}{m_0}\left(8m + 4\log\left(\frac{1}{\delta}\right) + 2\log\left(\frac{\mathrm{det}(V_t(\lambda))}{\lambda^{d_\Theta}}\right)\right)\right\} \leq \delta,
\end{equation}
for all $\lambda >0$.

Let $\Sigma = \frac{m_0}{\lambda} I \in \mathbb{R}^{d_\Theta\times d_\Theta}$ and let $h$ be the density of $\mathcal{N}(0,\Sigma)$. Then, for any fixed $x\in\mathcal{B}_m$ and $M_t(x,y)$ as in Lemma \ref{lemma:supmg}, define
\begin{align*}
    \bar M_t(x) &= \int_{\mathbb{R}^{d_\Theta}} M_t(x,y)h(y) = \frac{1}{\sqrt{(2\pi)^{d_\Theta} \mathrm{det}(\Sigma)}} \int_{\mathbb{R}^{d_\Theta}} \exp\left(y^\top S_t x - \frac{1}{2m_0}\|y\|_{V_t }^2 - \frac{1}{2}\|y\|_{\Sigma^{-1}}^2\right)dy.
\end{align*}

Notice that we can write
\[y^\top S_t x - \frac{1}{2m_0}\|y\|_{V_t }^2 - \frac{1}{2}\|y\|_{\Sigma^{-1}}^2 = \frac{1}{2}\|S_t x\|^2_{(\Sigma^{-1} + \frac{V_t }{m_0})^{-1}} - \frac{1}{2} \left\|y - \left(\Sigma^{-1} + \frac{V_t }{m_0}\right)^{-1}S_t x\right\|^2_{\Sigma^{-1} + \frac{V_t }{m_0}}.\]
Thus, by integrating out the Gaussian density, we get
\begin{align*}
\bar M_t(x) &= \exp\left(\frac{1}{2}\|S_t x\|^2_{(\Sigma^{-1} + \frac{V_t }{m_0})^{-1}}\right) \frac{1}{\sqrt{(2\pi)^{d_\Theta} \mathrm{det}(\Sigma)}} \int_{\mathbb{R}^{d_\Theta}} \exp\left(- \frac{1}{2} \left\|y - \left(\Sigma^{-1} + \frac{V_t }{m_0}\right)^{-1}S_t x\right\|^2_{\Sigma^{-1} + \frac{V_t }{m_0}}\right)dy\\
&= \exp\left(\frac{1}{2}\|S_t x\|^2_{(\Sigma^{-1} + \frac{V_t }{m_0})^{-1}}\right) \left(\frac{\mathrm{det} ((\Sigma^{-1} + \frac{V_t }{m_0})^{-1})}{\mathrm{det} (\Sigma)}\right)^{1/2}\\
&= \exp\left(\frac{m_0}{2}\|V_t^{-1/2}(\lambda)S_t x\|^2\right) \left(\frac{\lambda^{d_\Theta}}{\mathrm{det} (V_t(\lambda))}\right)^{1/2}.
\end{align*}

Now, by Lemma 20.3 in \citep{lattimore2020bandit}, since $M_t(x,y)$ is a supermartingale then $\bar M_t(x)$ is a non-negative supermartingale with $\bar M_0(x) = 1$. Thus, we can apply the maximal inequality to get
\begin{align}
\label{eq:St-bound-fixed-x}
&\mathbb{P}\left\{\exists t\in\mathbb{N} :  \log\bar M_t(x) \geq \log(1/\delta)\right\} \nonumber \\
&\quad = \mathbb{P}\left\{\exists t\in\mathbb{N} :  \frac{m_0}{2}\|V_t(\lambda)^{-1/2}S_t x\|^2 - \frac{1}{2}\log\left(\frac{\mathrm{det}(V_t(\lambda))}{\lambda^{d_\Theta}}\right) \geq \log(1/\delta)\right\} \leq \delta.
\end{align}

Inequality \eqref{eq:St-bound-fixed-x} is valid for all fixed $x\in\mathcal{B}_m$; to prove inequality \eqref{eq:St-bound}, we use a covering argument. Let $N_{\frac{1}{2},m}$ denote a $\frac{1}{2}$-net of $\mathcal{B}_m$, and note that we can make $|N_{\frac{1}{2},m}| \leq 5^m$. Then,
\[\|V_t(\lambda)^{-1/2}S_t\| = \max_{x\in\mathcal{B}_m} \|V_t(\lambda)^{-1/2}S_t x\| \leq 2 \max_{x\in N_{\frac{1}{2},m}} \|V_t(\lambda)^{-1/2}S_t x\|.\]
Therefore, we can apply a union bound to conclude that for all $s>0$,
\begin{align*}
    \mathbb{P}\left\{\exists t \in \mathbb{N}: \|V_t(\lambda)^{-1/2}S_t\|^2 \geq s\right\} &\leq \mathbb{P}\left\{\exists t \in \mathbb{N}: \max_{x\in N_{1/2,m}} \|V_t(\lambda)^{-1/2}S_t x\|^2_2 \geq \frac{s}{4}\right\}\\
    &\leq \sum_{x\in N_{1/2,m}}\mathbb{P}\left\{\exists t \in \mathbb{N}: \|V_t(\lambda)^{-1/2}S_t x\|^2_2 \geq \frac{s}{4}\right\}.
\end{align*}
By picking $s = \frac{1}{m_0}(8m + 4\log\frac{1}{\delta} + 2\log(\frac{\mathrm{det}(V_t(\lambda))}{\lambda^{d_\Theta}})) \geq \frac{1}{m_0}(4\log\frac{5^m}{\delta} + 2\log(\frac{\mathrm{det}(V_t(\lambda))}{\lambda^{d_\Theta}}))$ and applying Equation~\eqref{eq:St-bound-fixed-x}, we get
\[\mathbb{P}\left\{\exists t \in \mathbb{N}: \|V_t(\lambda)^{-1/2}S_t\|^2 \geq \frac{1}{m_0}\left(8m + 4\log\left(\frac{1}{\delta}\right) + 2\log\left(\frac{\mathrm{det}(V_t(\lambda))}{\lambda^{d_\Theta}}\right)\right)\right\} \leq \sum_{x\in N_{1/2,m}}\frac{\delta}{5^m} \leq \delta.\]
This completes the proof of inequality \eqref{eq:St-bound}.

It remains to relate this bound to the definition of $\mathcal{C}_t$. We can write
\[\hat \mu_t - \mu_* = V_t(\lambda)^{-1}S_t + V_t(\lambda)^{-1} V_t  \mu_* - \mu_*,\]
and therefore
\begin{align*}
    \|V_{t}(\lambda)^{1/2}(\hat \mu_{t} - \mu_*)\| &= \|V_t(\lambda)^{-1/2}S_t + V_t(\lambda)^{1/2}(V_t(\lambda)^{-1}V_t  - I)\mu_*\|\\
    &\leq \|V_t(\lambda)^{-1/2}S_t\| + \sqrt{\|\mu_*^\top ( V_t(\lambda)^{-1}V_t  - I) V_t(\lambda) (V_t(\lambda)^{-1}V_t   - I)\mu_*\|}\\
    &= \|V_t(\lambda)^{-1/2}S_t\| + \sqrt{\lambda} \sqrt{\|\mu_*^\top (I - V_t(\lambda)^{-1}V_t )\mu_*\|}\\
    &= \|V_t(\lambda)^{-1/2}S_t\| + \sqrt{\lambda} \|\mu_*\|,
\end{align*}
where the second equality follows by writing $V_t  = V_t(\lambda) - \lambda I$. Note additionally that by $\max\{\|\theta\|:\theta\in\Theta\}\leq 1$ and the AM-GM inequality,
\[\mathrm{det}(V_t(\lambda)) \leq \left(\frac{1}{d_\Theta}\mathrm{trace}V_t(\lambda)\right)^{d_\Theta} \leq \left(\frac{d_\Theta\lambda + t}{d_\Theta}\right)^{d_\Theta}.\]
Applying Equation~\eqref{eq:St-bound}, setting $\delta = \frac{1}{T^2}$ and $\lambda = \frac{1}{m_0}$ completes the proof.
\end{proof}

\subsection{Regret bound on the clean event}

The place where the structure of the performative risk comes into play is the following lemma, where we relate the suboptimality of the deployed model $\theta_t$ to properties of the confidence set $\mathcal{C}_t$. 
\begin{lemma}
\label{lemma:confidencelin}
Suppose that the clean event \eqref{eq:clean_linucb} holds. Then, we can bound the suboptimality of $\theta_t$ by 
\[\Delta(\theta_t) \le \min\left\{1, L_z \sup_{\mu, \mu' \in \mathcal{C}_t} \|(\mu - \mu')^\top \theta_t\|\right\}. \]
\end{lemma}
\begin{proof}
In what follows, all expectations are taken only over a sample $z_0\sim\cD_0$ independent of everything else (i.e., all other random quantities are conditioned on).

Since the loss is bounded, we know $\Delta(\theta_t) \leq 1$. For the other bound, notice that
\[\Delta(\theta_t) = \mathbb{E} \ell(z_0 + \mu_*^\top\theta_t;\theta_t) - \mathbb{E} \ell(z_0 + \mu_*^\top \thetaPO;\thetaPO).\]
By the definition of the algorithm and the clean event, we can lower bound the second term $\mathbb{E} \ell(z_0 + \mu_*^\top \thetaPO;\thetaPO)$ as follows: 
\[\mathbb{E} \ell(z_0 + \mu_*^\top \thetaPO;\thetaPO) \geq \PRLB(\thetaPO)\geq \PRLB(\theta_t) = \mathbb{E} \ell(z_0 + \tilde\mu_t^\top\theta_t;\theta_t),\]
for some $\tilde \mu_t \in \mathcal{C}_t$. This means that: 
\[\Delta(\theta_t) \leq \mathbb{E} \ell(z_0 + \mu_*^\top \theta_t;\theta_t) - \mathbb{E} \ell(z_0 + \tilde \mu_t^\top \theta_t;\theta_t).\]
To finish, we use Lipschitzness of the loss to upper bound this by $L_z\|(\mu_* - \tilde \mu_t)^\top \theta_t\|$. Using the clean event, we can further upper bound this by $L_z \sup_{\mu, \mu' \in \mathcal{C}_t} \|(\mu - \mu')^\top \theta_t\|$ as desired. 
\end{proof}

We now use this bound on the suboptimality of deployed models, along with the structure of the confidence sets, to bound the regret on the clean event. 
\begin{lemma}
\label{lemma:reg_analysis}
Let $1\leq\beta_1\leq\beta_2\leq\dots\beta_T$ and assume that the loss $\ell(z;\theta)$ is $L_z$-Lipschitz in $z$. Assume that the event
\[\mu_*\in \mathcal{C}_t \subseteq \left\{\mu\in\mathbb{R}^{d_\Theta\times m}:\left\|V_{t-1}^{1/2}\left(\frac{1}{m_0}\right)(\mu - \hat \mu_{t-1})\right\|^2 \leq \beta_t\right\}\]
holds true, for all $2 \le t \le T$. Then, on this event, Algorithm \ref{alg:location-families} satisfies:
\[\sum_{t=1}^T \Delta(\theta_t) = \tilde \cO\left(1 + \sqrt{d_\Theta T\beta_T\log\left(\frac{d_\Theta + T m_0 }{d_\Theta}\right)} \max\{L_z, 1\} \right).\]
\end{lemma}

\begin{proof}
As in the proof of Lemma \ref{lemma:confidencelin}, 
all expectations are taken only over a sample $z_0\sim\cD_0$ independent of everything else (i.e., all other random quantities are conditioned on).

First, we separately bound the regret of the first step as $\cO(1)$, using the fact that the loss is bounded in $[0,1]$. 

For the remainder of the steps, we apply Lemma \ref{lemma:confidencelin} to upper bound $\Delta(\theta_t)$. Using this, coupled with structure of $\mathcal{C}_t$, we can obtain the following upper bound, for any $\lambda > 0$:
\begin{align*}
  \Delta(\theta_t) &\le \min\left\{1, L_z \sup_{\mu, \mu' \in \mathcal{C}_t} \|(\mu - \mu')^\top \theta_t\|\right\} \\
  &\le \min \left\{1, L_z \sup_{\mu, \mu' \in \mathcal{C}_t}  \|(\mu - \mu')^\top V_{t-1}^{1/2}(\lambda)\|\cdot\|V_{t-1}^{-1/2}(\lambda)\theta_t\|\right\} \\
    &\leq \min\left\{1, 2L_z \sqrt{\beta_t}\|V_{t-1}^{-1/2}(\lambda)\theta_t\| \right\} \\
    &\le 2\sqrt{\beta_T}\min\left\{1,L_z \|V_{t-1}^{-1/2}(\lambda)\theta_t\|\right\},
\end{align*}
where the last line uses the fact that $\beta_T \geq \max\{1,\beta_t\}$.

By the Cauchy-Schwarz inequality,
\begin{align*}
  \sum_{t=2}^T \Delta(\theta_t) &\leq \sqrt{T\sum_{t=2}^T \Delta(\theta_t)^2}\\
  &\leq 2 \sqrt{T\beta_T \sum_{t=2}^T \min\left\{1, L_z^2\|V_{t-1}^{-1/2}(\lambda)\theta_t\|^2\right\}} \\
  &\leq 2 \sqrt{T\beta_T \sum_{t=2}^T \min\left\{1, \max\{1, L_z^2\}  \|V_{t-1}^{-1/2}(\lambda)\theta_t\|^2\right\}}   \\ 
&\leq 2 \sqrt{T \max\{1, L_z^2\} \beta_T \sum_{t=2}^T \min\left\{1, \|V_{t-1}^{-1/2}(\lambda)\theta_t\|^2\right\}}  \\
&= 2 \max\{1, L_z\}  \sqrt{T \beta_T \sum_{t=2}^T \min\left\{1, \|V_{t-1}^{-1/2}(\lambda)\theta_t\|^2\right\}}.
\end{align*}
Finally, we use Lemma 19.4 in \citep{lattimore2020bandit} that says
\[\sum_{t=2}^T \min\left\{1, \|V_{t-1}^{-1/2}(\lambda)\theta_t\|^2\right\} \leq 2d_\Theta \log\left(\frac{\mathrm{trace}V_0(\lambda) + T}{d_\Theta  \mathrm{det}(V_0(\lambda))^{1/d_\Theta }}\right) = 2d_\Theta \log\left(\frac{d_\Theta \lambda + T}{d_\Theta  \lambda}\right).\]

Using this expression in the equation above and setting $\lambda = \frac{1}{m_0}$ yields the final result.
\end{proof}

\subsection{Proof of Theorem \ref{thm:lin_bound}}

We take $\sqrt{\beta_t} = \max\left\{1, \sqrt{\frac{1}{m_0}} M_* + \sqrt{\frac{8m + 8\log T + 2d_\Theta\log\left(\frac{d_\Theta + t m_0}{d_\Theta }\right)}{m_0}}\right\}$. By the constraint that $m_0 = o(\log T)$, we see that second branch dominates over the first one and so, for large enough $T$, $\sqrt{\beta_t} =  \sqrt{\frac{1}{m_0}} M_* + \sqrt{\frac{8m + 8\log T + 2d_\Theta\log\left(\frac{d_\Theta + t m_0}{d_\Theta }\right)}{m_0}}$. Lemma \ref{lemma:conf_regions} shows that: 
\[\mu_* \in \mathcal{C}_t \subseteq \left\{\mu\in\mathbb{R}^{d_\Theta \times m} : \left\|V_{t-1}^{1/2}\left(\frac{1}{m_0}\right)(\mu-\hat \mu_{t-1})\right\|^2 \leq \beta_t\right\}.\] 
Moreover, the contribution of the complement of the clean event to the overall regret is negligible. Plugging this choice of $\beta_t$ into the bound of Lemma \ref{lemma:reg_analysis} completes the proof of Theorem~\ref{thm:lin_bound}.

\section{Further details on zooming dimension}

\subsection{Discussion of zooming dimension definitions}\label{app:zooming}

We note that Definition~\ref{def:zdim} slightly differs from the definition presented in~\citep{kleinberg2008multi}. The statement of Definition~\ref{def:zdim} eases the comparison of the zooming algorithm of Kleinberg et al. to our new algorithm. 

First, we introduce a multiplier $\alpha $ to emphasize that the zooming dimension implicitly depends on the Lipschitz constant of the problem (assumed to be fixed and equal to $1$ by Kleinberg et al.), which can be smaller when we make full use of performative feedback. 

Second, Definition~\ref{def:zdim} is slightly more conservative in two ways. One is that we intersect a cover of any subset of $\{\theta:\Delta(\theta)\leq 16\alpha s\}$ with $\{\theta:16\alpha r < \Delta(\theta) \leq 32\alpha r\}$, rather than directly take a cover of the latter set. The other one is that we take a supremum over all covers with radius coarser than $r$, i.e. $s\in[r,1]$, instead of only $r$. These differences are minor technicalities that we do not expect to alter the zooming dimension in a meaningful way, neither formally nor conceptually. 

Lastly, rather than requiring the size of the relevant set of points to be at most of order $(1/s)^{d}$, we require the size to be at most of order $(3/s)^{d}$. In this regard, Definition \ref{def:zdim} is less conservative than the zooming dimension in \citep{kleinberg2008multi}. We make this modification so that for the Euclidean ball of dimension $d_\Theta$ of radius $1$, which contains $\Theta$, the zooming dimension is guaranteed to be at most $d_\Theta$. This would not be true without the factor of $3$. We note that the analysis of adaptive zooming in \citep{kleinberg2008multi} can be modified in a straightforward way to allow for this change, only altering constant factors in the regret bounds.

\subsection{Gains of sequential zooming dimension}\label{appendix:seqzooming}

We provide an example where the sequential zooming dimension is strictly smaller than the zooming dimension.
\begin{example}
\label{ex:seqzoom}
Suppose that model parameters are 2-dimensional, $L_z \epsilon = 1/32$, and the distribution map is a fixed distribution: $\PR(\theta) = \DPR(\theta’, \theta)$ for all $\theta, \theta'$. Let $\theta_0 = 0, \theta_1 = [1/2, 0], \theta_2 = [1/4, \sqrt{3}/4]$. Suppose that $\PR(\theta_0) = 0$, $\PR(\theta_1) = 1/8$, $\PR(\theta_2) = 15/64$, and $\PR(\theta) = 1$ otherwise. 
\end{example}

\begin{lemma}
In Example \ref{ex:seqzoom}, the ($L_z\epsilon$)-zooming dimension is at least $d \ge 0.39$, and the ($L_z\epsilon$)-sequential zooming dimension is at most $d \le \log_6(1.5) \approx 0.23$.
\end{lemma}
\begin{proof}

We begin by observing that for all $0 < s \le 1$, it holds that: 
\[\left\{ \theta \mid \Delta(\theta) \le 16 L_z \epsilon s \right\} \subseteq \left\{ \theta \mid \PR(\theta) \le 16 L_z \epsilon \right\} = \left\{ \theta \mid \PR(\theta) \le 1/2 \right\} = \left\{\theta_0, \theta_1, \theta_2\right\}.\] Note that $\theta_0$ achieves the optimal performative risk and thus does not appear in any suboptimality band $\left\{\theta \mid 16 L_z \epsilon r \le \PR(\theta) < 32 L_z \epsilon r  \right\}$. 

First, we show that the zooming dimension is at least $\log_6(2) \approx 0.39$. Let $s = 1/2 - \epsilon$ for $\epsilon$ sufficiently small. Consider the set $\left\{ \theta \mid \Delta(\theta) \le 16 L_z \epsilon s \right\} = \left\{\theta_0, \theta_1, \theta_2\right\}$. A minimal covering of the set will necessarily consist of all three points $\left\{\theta_0, \theta_1, \theta_2\right\}$. We see that the suboptimality band $\left\{\theta \mid 16 L_z \epsilon r \le \PR(\theta) < 32 L_z \epsilon r  \right\}$ for $r = 1/4$ contains $\left\{\theta_1, \theta_2\right\}$. Taking $\epsilon \rightarrow 0$, we see that the zooming dimension is at least $\log_{6}(2)$. 

Next, we show that the sequential zooming dimension is $d \le \log_6(1.5) \approx 0.23$. Let $\mathcal{C}$ be a minimal covering of a subset of $\left\{ \theta \mid \Delta(\theta) \le 16 L_z \epsilon s \right\}$. For $s \ge 1/2$, we see that $\left\{ \theta \mid \Delta(\theta) \le 16 L_z \epsilon s \right\} = \left\{\theta_0, \theta_1, \theta_2 \right\}$, and  $\mathcal{C}$ contains at most $1$ point. If $s < 1/2$, then $\mathcal{C}$ might contain up to 3 points. If $\mathcal{C}$ does not contain both $\theta_1$ and $\theta_2$, then any suboptimality band $\left\{\theta \mid 16 L_z \epsilon r \le \PR(\theta) < 32 L_z \epsilon r  \right\}$ contains at most 1 point from $\mathcal{C}$. If $\mathcal{C}$ contains both $\theta_1$ and $\theta_2$, then we leverage the sequential properties of the sequential zooming dimension. We claim that pulling $\theta_1$ first results in $\theta_2$ being eliminated. Notice that if $\theta_1$ is pulled first, then $\PRmin$ will be equal to $1/8$. $\PRLB$ will be equal to $\DPR(\theta_1, \theta_2) + L_z \epsilon ||\theta_1 - \theta_2|| = \PR(\theta_2) + L_z \epsilon ||\theta_1 - \theta_2|| = 15/64 - 1/64 = 14/64 = 7/32$. We see that $\PRmin + 4L_z \epsilon s \le 1/8 + 1/16 = 3/16$. Since $7/32 > 3/16$, we see that  $\theta_1$ will be eliminated. This means that in expectation, at most $1.5$ arms are pulled. This yields the desired bound.

\end{proof}

\section{Details of numerical illustrations}
\label{app:figures}

For the purpose of the illustrations in Figure~\ref{fig:confidence}, Figure~\ref{fig:compare-discarding}, and Figure~\ref{fig:sequential} we use a one-dimensional example where $\theta\in \mathbb{R}$. The performative effects are modeled by a linear shift, i.e., 
\[\DPR(\phi,\theta) = f(\theta)+\alpha\phi,\]
where $f$ is a multi-modal function illustrated in the respective figures and specified as
\[f(\theta)= c_0 \cos(c_1 \theta) +c_2 \sin(c_3 (\theta-c_4)).\]
The shaded gray area in the figures illustrates the confidence sets computed as 
\begin{align*}
    \PR_\text{LB}(\theta)=\max_{\theta'\in \cS}\PR(\theta')- L_\PR \|\theta-\theta'\|,\quad \quad \PR_\text{UB}(\theta)=\min_{\theta'\in \cS}\PR(\theta')+ L_\PR \|\theta-\theta'\|
 \end{align*}
for the baseline approach, and as 
\begin{align*}
    \PR_\text{LB}(\theta)=\max_{\theta'\in \cS}\;\DPR(\theta',\theta)- L_\phi \|\theta-\theta'\|,\quad \quad 
    \PR_\text{UB}(\theta)=\min_{\theta'\in \cS}\;\DPR(\theta',\theta)+ L_\phi \|\theta-\theta'\|
 \end{align*}
for the performative confidence bounds. We use $\cS:=\{\theta_1,\theta_2\}$ as shown in the figures. The Lipschitz constant $L_\PR$ of the performative risk $\PR(\theta)=\DPR(\theta,\theta)$ is evaluated numerically for each figure.
 
For Figure~\ref{fig:confidence} and Figure~\ref{fig:compare-discarding} we use the following parameters: $c_0=-1$, $c_1=0.7$, $c_2=0.3$, $c_3=3$, $c_4=0.5$, $\alpha=1$, and a conservative Lipschitz bound $L_\phi=1.6$ for the performative confidence bounds  and $L_\PR=3.8$ for the performative risk.
 
For Figure~\ref{fig:sequential} we use $c_0=-3$, $c_1=1$, $c_2=0.9$, $c_3=3$, $c_4=0.5$, $\alpha=0.5$, and a conservative Lipschitz bound $L_\phi=1.3$ for illustrating the performative confidence bounds. If exact knowledge of the shifts were available these bounds could be made even tighter.

\end{document}